\documentclass[generic,reqno]{imsart}
\setattribute{journal}{name}{generic}

\RequirePackage[OT1]{fontenc}
\RequirePackage{amsthm,amsmath,natbib}
\RequirePackage{hypernat}
\RequirePackage{comment, listings}
\RequirePackage{pzccal, url}
\RequirePackage{graphicx,subfigure,latexsym,amssymb}
\RequirePackage{float,epsfig,multirow,rotating}
\RequirePackage{upgreek,wrapfig}
\startlocaldefs

\numberwithin{equation}{section}
\theoremstyle{plain}
\newtheorem{theorem}{Theorem}[section]

\newtheorem{definition}{Definition}[section]
\newtheorem{remark}{Remark}[section]
\newtheorem{proposition}{Proposition}[section]

\DeclareMathOperator\erf{erf}

\newcommand{\bSigma}{\boldsymbol{\Sigma}}

\newcommand{\bPsi}{\boldsymbol{\Psi}}

\newcommand{\bC}{\boldsymbol{C}}

\newcommand{\bV}{\boldsymbol{V}}

\newcommand{\bH}{\boldsymbol{H}}

\newcommand{\bP}{\boldsymbol{P}}

\newcommand{\bR}{\boldsymbol{R}}
\newcommand{\bW}{\boldsymbol{W}}

\newcommand{\bx}{\boldsymbol{x}}
\newcommand{\bX}{\boldsymbol{X}}

\newcommand{\bY}{\boldsymbol{Y}}
\newcommand{\bZ}{\boldsymbol{Z}}

\newcommand{\bone}{\boldsymbol{1}}

\endlocaldefs

\begin{document}

\begin{frontmatter}
\title{Learning Random Geometric Graphs Drawn in Probabilistic Metric Spaces} 


\runtitle{Learning RGGs in Probabilistic Spaces}

\begin{aug}
\author{                                                                      
  {\fnms{Dalia}\snm{ Chakrabarty}\thanksref{m2}\ead[label=e2]{dalia.chakrabarty@york.ac.uk}},
  {\fnms{Kangrui}\snm{ Wang}\thanksref{m1}\ead[label=e1]{kangrui.wang@turing.ac.uk}},
  {\fnms{Chuqiao}\snm{ Zhang}\thanksref{m2}\ead[label=e3]{joe.zhang@brunel.ac.uk}},
  {\fnms{Ye}\snm{ Liu}\thanksref{m4}\ead[label=e4]{y.liu@hwallingford.com}}  
},

\runauthor{Chakrabarty et. al} 

\affiliation{University of York, Alan Turing Institute, HR.Wallingford}

\address{\thanksmark{m1} Alan Turing Institute\\
London NW1 2DB\\
U.K.\\                                                                        
\printead*{e1}                                                                
}
\address{\thanksmark{m2} Department of Mathematics\\
University of York\\
York YO1 7DD\\
U.K.\\                                                                        
\printead*{e2}
\printead*{e3}
}
\address{\thanksmark{m4} 
Howbery Park, Wallingford\\
OX10 8BA\\
U.K.\\                                                                        
\printead*{e4}                                                                
}

\end{aug}

\begin{abstract} {We present a new data-driven learning of a Random Geometric Graph (RGG) of a multivariate dataset, where the graph is drawn in a probabilistic metric space. This graph learning works for generic datasets, irrespective of the type of the observables; their probability distributions; or size of the data. We identify a metric of the space that the graph is drawn in, as a probability distribution of a random variable that we introduce, namely, a variable that represents the disparity between the connectedness of two vertices of the graph, and the correlation between the two random variables that are attached to the respective vertex. It is the closed-form {\it{cdf}} of this disparity variable that we advance as the distance function of the host space of the learnt RGG, such that the edge exists between any two nodes, if this inter-nodal distance falls short of a chosen cutoff probability. Drawing the RGG in this probabilistic space leads to the graph being an Soft RGG, such that any edge - if it exists - exists with an identified probability. We forward a simple Rejection Sampling-based technique for learning the probability of any edge. The expected degree distribution of a vertex of this RGG is identified as local, and dependent on the inter-observable correlation matrix. If said correlation matrix is not known, it can be learnt given the data, using its closed-form posterior probability density function, that we forward. We illustrate our graph learning method by learning multiple RGGs of highly multivariate real datasets.
}
\end{abstract}



\end{frontmatter}

\section{Introduction}
\label{sec:intro}
\noindent
Graphs of complex multivariate datasets manifest intuitive
illustrations of correlation structures of data, and are of pan-disciplinary
interest \citep{whittaker, large_net, carvalhowest, dipankar, plos2007, anand}.

We present a new method for learning a probabilistic graph of a given multivariate dataset, using the correlation structure of this data, where the graph variable in our consideration is a Random Geometric Graph or RGG
\citep{penrose,penrose_2016,Giles2016}, that is drawn in probabilistic metric
space, \citep{menger, sklar}, s.t. the distance between a pair of nodes, is a
probability distribution over positive support. We will show that such an RGG
is in fact, a Soft
Random Geometric Graph or SRGG \citep{dettman_2018, penrose_2016}. We recall
that the probability for an edge to exist in an SRGG is the ``connection function'', that takes as its input, the
inter-nodal distance, while in an RGG, an edge exists if the distance between the nodes - that straddle the edge - falls short of a cutoff. Said cutoff is a probability in the RGG drawn in a probabilistic metric space.

The inter-nodal distance of this RGG is advanced as a metric in the probabilistic metric space that the graph is drawn in, and this distance is used to compute the expected degree distribution of the learnt graph, as local, and dependent on the inter-observable correlation matrix of the given multivariate dataset that comprises values of such observables.
We provide a closed-form posterior probability density of this
inter-observable correlation matrix.
We also provide a
closed-form probability of the RGG variable, computed given this
correlation matrix. 
Our graph learning is fast, and is not constrained by the nature of the observables or that of the correlation structure of the data at hand. The size of the data is also not a constraint on our graph learning - we perform an empirical illustration of learning an RGG with a large ($\sim 7000$) number of nodes.



We start with a preliminary recollection of the basic definitions in
Section~\ref{sec:prelims}. In Section~\ref{sec:corr}, we present the closed-form density of the inter-observable correlation matrix of the given dataset. Thereafter, we develop the closed-form probability of an edge of the RGG variable (in Section~\ref{sec:graph}), and identify the inter-nodal distance of this graph, as a distance function in the probabilistic metric space in which the RGG is drawn. Thus, this distance function is shown to be a probability distribution over a positive support (Section~\ref{sec:distiscdf}). This RGG is shown to be an SRGG in Section~\ref{sec:issrgg}, where the connection function of this SRGG is shown to be a function of our identified inter-nodal distance. We compute the expected degree distribution of our learnt RGG in Section~\ref{sec:point} and subsequently, declare the host probabilistic metric space in which we draw the RGG. The technique for learning the RGG is delineated in Section~\ref{sec:graph}. Then in Section~\ref{sec:algo} we present the Bayesian inference that is undertaken to learn the graph alone - when the correlation of the given dataset is known; else, the graph and this correlation are learnt simultaneously with a Metropolis-with-a-2-block-update.
Learning of the RGG of a cuboidally-shaped data is discussed in Section~\ref{sec:hypergraph2}. Thereafter, we present an illustration on real data - to learn the large network of diseases that afflict humans, correlated by phenotypes (in Section~\ref{sec:disease}). Empirical illustration on simulated data is included in the Supplement, which includes demonstration of the effect of noise in the data on learning the RGG,

\section{Preliminaries}
\label{sec:prelims}
In this section we provide the background on concepts that we use in the
learning of the graph of a multivariate data set.
\subsection{RGGs and SRGGs}
\label{subsec:srgg}
\noindent
We seek a graph variable defined on vertex set $\bV =\{1,2,
\ldots,p\}$, with the random variable $X_i$ attached to the $i$-th node,
$\forall i\in\bV$ and $\forall X_i\in{\cal X}\subseteq{\mathbb R}$. Edges form
independently in this graph and there are no self-loops. Let
$G_{i,j}=g_{i,j}\in\{0,1\}$ be the edge variable between the $i$-th and $j$th nodes,
$i,j\in [M]$, where $[M] := \{i,j: i < j; i,j\in\bV\}$.
\begin{definition}
  \label{defn:rgg}
In an RGG, $G_{i,j}\sim$Poisson Point Process
with its realisation $g_{i,j}=1$ $\iff$ $D(X_i, X_j) < \tau$, for a chosen cutoff
$\tau \geq 0$, where distance between the $i$-th and $j$-th nodes is
$D(X_i,X_j)$, with $D:{\cal X}\times{\cal
  X}\longrightarrow{\mathbb R}_{\geq 0}$.
\end{definition}
We note that $G_{i,j}$ is the binary random variable that we assign a value of 1 to, if the $i$-th and $j$-th nodes are connected by an edge in the graph; if the edge does not exist, then we set $G_{i,j}=0$.
\begin{definition}                                                             
An SRGG defined on vertex set $\bV$ is
s.t. $\Pr(G_{i,j}=1)$
is given by the connection function
$\phi(d)$, where distance function between the $i$-th and $j$-th nodes is $d$. Thus, $\phi:{\mathbb R}_{\geq
  0}\longrightarrow[0,1]$ is the connection function of this SRGG, i.e. edge exists between the $i$-th and $j$-th nodes in this SRGG, with probability $\phi(d)$.
\end{definition}

\subsection{Partial correlation matrix}

Let the given multivariate dataset ${\bf D}$, comprise $n$ observations of each of $p$ observables $X_1,X_2,\ldots,X_p$, such that the $i$-th row of ${\bf D}$ is: $x_1^{(i)}, x_2^{(i)}, \ldots, x_p^{(i)}$, where $i\in\{1,2,\ldots,n\}$. We standardise data on the $j$-th observable, by the estimated mean and standard deviation of the sample $\{x_j^{(1)}, x_j^{(2)}, \ldots, x_j^{(n)}\}$, s.t. the dataset comprising $n$ values of each of the standardised observables, is ${\bf D}_S$. (Later in Section~\ref{sec:hypergraph2} we will discuss the general case of $\bX_j$ a $d$-dimensional random variable, $\forall j\in\{1,\ldots,p\}$). Here $j\in\{1,\ldots,p\}$. Then the inter-observable correlation matrix of the data ${\bf D}_S$ is $\bSigma_C^{(S)}=[corr(X_i, X_j)]$.

\begin{definition}
{
For the precision
matrix $\bPsi = [\psi_{i,j}] :=\left(\bSigma_C^{(S)}\right)^{-1}$,
the matrix of absolute partial correlations is $\bR=[\rho_{i,j}]$, with
\begin{equation}
\rho_{i,j} = {\Big\vert} -\displaystyle{\frac{\psi_{i,j}}{\sqrt{\psi_{ii} \psi_{jj}}}}{\Big\vert},\quad
i\neq j,{\mbox{ and }}\rho_{ii}=1 {\mbox{ for }} i=j. 
\label{eqn:6}
\end{equation}
}
\end{definition}
Here we want to learn the graph of dataset ${\bf D}$,
using the known absolute partial correlation $\rho_{i,j}$ between $X_i$ and $X_j$ $\forall i,j\in[M]$.
If however, we wish to learn the graph given the known 
(absolute) correlation $\vert corr(X_i, X_j)\vert$ between this variable pair - we only have to replace
$\rho_{i,j}$ in the discussion below, with $\vert corr(X_i, X_j)\vert$, $\forall i,j\in[M]$. Keeping this in mind, in the discussions below, we will include the possibilities of graph learning given the correlation or the partial correlation, by referring to ``(partial) correlation'' throughout, unless otherwise stated.

Below, the developed methodology considers that the (partial) correlation matrix of the considered dataset is known.

\subsection{RGG drawn in probabilistic metric space: an SRGG} 
We construct an RGG variable in a probabilistic metric space. This renders any
edge of this graph a probability distribution (over positive support). Then
the inter-nodal distance is anticipated to be a distribution with positive
support, and this distribution will be shown to be 
conditional on the observation that is the absolute partial correlation $\rho_{i,j}$ between $X_i$ and $X_j$. 


To reflect this conditioning, we update the notation for distance between $i$-th and $j$-th nodes to
$D(X_i, X_j; \rho_{i,j})$, which we will sometimes abbreviate to $D_{i,j}$. 


We learn edge $G_{i,j}$ given the value $\rho_{i,j}$ of the absolute partial correlation
variable $R_{i,j}$, between $X_i$ and $X_j$, $i,j\in [M]$. Thus, $R_{i,j} = \rho_{i,j}\in[0,1]$.
%

In this section we discuss the learning of the graph of dataset ${\bf D}$, the (partial) correlation matrix of which we treat as known. As a supplement to this current agenda, later in Section~\ref{sec:corr}, we will discuss how we can learn the correlation matrix of the given dataset by developing the closed-form posterior probability density of the correlation matrix. 




In Section~\ref{sec:issrgg}, we will show that by constructing the random graph as an RGG drawn in a
probability metric space, the graph is rendered an SRGG.
We will prove this using the identified distance function in the space in which we draw the graph, and learn using the (partial) correlation structure of the given multivariate dataset.

\subsection{Disparity between two variables}
\label{subsec:menger}
\noindent
We embed the graph variable in a probabilistic metric space \citep{menger}, where in
such a space, to any pair of points, we assign a probability distribution over
positive support. Thus, 
to the $i,j$-th pair of nodes in the graph, we assign a probability distribution $F_{S(X_i,X_j)}(s)$ of the {\it{disparity}} variable
$S(X_i,X_j)$ that is introduced in Definition~\ref{defn:intro}.
%


\begin{definition}
  \label{defn:intro}
{The disparity random variable is defined as
$$S(X_i, X_j) \equiv S_{i,j} := \vert G_{i,j} - R_{i,j}\vert.$$}
\end{definition}
In our work, disparity $S_{i,j}$ measures the
difference between
the known (partial) {\it{correlation}} between $X_i, \: X_j$, and the 
{\it{connectedness}} of the $i$-th and $j$-th nodes of the graph, where the absolute partial correlation variable is $R_{i,j}$, and the binary edge variable $G_{i,j}$ represents the connectedness between the $i$-th and $j$-th nodes. This holds $\forall i,j\in[M]$. Then the disparity $S_{i,j} = s\in[0,1]$, $\forall i,j\in[M]$.


Given the sparse information that we have on the behaviour of the disparity variable $S_{i,j}$, we recall (using its value $s:=\vert g_{i,j} - \rho_{i,j}\vert$) that
  \begin{enumerate}
  \item[---]if the given absolute value $\rho_{i,j}$ of the (partial) correlation between $X_i$ and $X_j$ is high, (i.e. close to 1), then the probability for the edge to exist between the $i$-th and $j$-th nodes is high; probability for $G_{i,j}$ to attain the value 0 then, is low. Thus, the probability for $S_{i,j}$ to attain a low value is higher than for disparity to attain a higher value. 
  \item[---]again, if the given $\rho_{i,j}$ is close to 0, the probability for $G_{i,j}$ to be 0 is high, while that for $G_{i,j}=1$ is low. Thus, again, we find that probability of $S_{i,j}$ at a low $s$ is higher than that at a higher $s$.
  \end{enumerate}
  Using this intuition, we model the probability density of $S_{i,j}$ as a decreasing function of $s\in[0,1]$. However, the scale with which this density decreases with $s$, is an unknown, i.e. this scale is a real random variable. We denote such a squared scale, (or a variance variable) to be $\upsilon_{i,j}$.
  There is no available observation to further constrain the model of the joint density of $S_{i,j}$ and $\upsilon_{i,j}$, than what we have motivated above. Hence we model this joint density of the disparity variable $S_{i,j}$ and the variance variable $\upsilon_{i,j}$, at $S_{i,j}=s$ and $\upsilon_{i,j}=\upsilon$ as: 
\begin{eqnarray}
  f_{S_{i,j},\upsilon_{i,j}}(s,\upsilon)
  &=&
  K\displaystyle{
    \frac{1}{\sqrt{2\pi\upsilon}}\exp\left(-\frac{s^2}{2\upsilon}\right)},{\text{ if }} s\in[0,1]; v\in{\mathbb R}_{> 0}\nonumber \\
  &=& 0 {\text{ otherwise}},\nonumber
  \end{eqnarray}
  $\forall i,j\in [M]$, where $K > 0$ is a constant that we compute below.

Also, lack of further information motivates our assumption that the prior on the variance variable $\upsilon_{i,j}$ is Uniform[0,1]. The above holds $\forall i,j\in[M]$. For other models of this joint density, Bayesian inference on the edge variable is likely to converge in mean.

Keeping in mind our ulterior aim of learning the edge given the known inter-observable correlation, we will need to marginalise out the variance variable from this modelled joint density of variance and disparity, (where the latter disparity bears information on the sought edge). 

Since we assume $\upsilon\sim Uniform[0,1]$, the marginal density of the disparity variable $S_{i,j}$ is 
\begin{eqnarray}
  f_{S_{i,j}}(s) &=& 
  K\displaystyle{\int\limits_{0}^1 \frac{1}{\sqrt{2\pi\upsilon}}\exp\left(-\frac{s^2}{2\upsilon} \right)d\upsilon},\:\:
  s\in[0,1],\nonumber \\
  &=& 0 {\text{ otherwise}},\nonumber
\end{eqnarray}  
$\forall i,j\in[M]$, where $K >0$ is a global scale. Then given that $s \in[0,1]$, this marginalisation over the variance parameter gives
    \begin{eqnarray}
      f_{S_{i,j}}(s) &=& \displaystyle{\frac{K s\Gamma\left(-\frac{1}{2}, \frac{s^2}{2}\right)}{2\sqrt{\pi}}} \nonumber \\
      &=& K\left[\sqrt{\frac{2}{\pi}}\exp\left(-\frac{s^2}{2}\right) + s \erf\left(\frac{s}{\sqrt{2}}\right) -s\right],
      \label{eqn:pdf_s}
    \end{eqnarray}
for $s\in[0,1]$; else $f_{S_{i,j}}(s)=0$. This holds $\forall i,j\in[M]$.

    Since $f_S(s)$ is a density, the constant $K$ is s.t
\begin{eqnarray}
  K &=& \displaystyle{\left[\int_{0}^1 \left(\sqrt{\frac{2}{\pi}}\exp\left(-\frac{s^2}{2}\right) + s\: {\erf}\left(\frac{s}{\sqrt{2}}\right) -s\right) ds\right ]^{-1}} \nonumber \\
  &=& \displaystyle{\left[\erf\left(\frac{1}{\sqrt{2}}\right) -\frac{1}{2} + \frac{1}{\sqrt{2\pi e}} \right]^{-1}}
  \label{eqn:k}
  \end{eqnarray}    



We will show below that when the absolute partial correlation between $X_i$ and $X_j$ is known to be $\rho_{i,j}$, the distance between the $i$-th and $j$-th points in this host space is $D(X_i, X_j; \rho_{i,j}) = F_{S(X_i,X_j)}(s)$,
which is the {\it{cdf}} of the disparity computed at a value $s := \vert g_{i,j} - \rho_{i,j}\vert$. We will compute said {\it{cdf}} by invoking the aforementioned {\it{pdf}} of the disparity presented in Equation~\ref{eqn:pdf_s}. Then, given the multivariate dataset (that leads to the matrix $\bR=[\rho_{i,j}]$ of absolute inter-observable (partial) correlations), the learnt random graph variable is ${\cal G}_{S, \bV}(\bR,\tau)$ that is defined on vertex set $\bV$ s.t. the edge exists between the $i$-th and $j$-th nodes $\iff$ $D(X_i, X_j; \rho_{i,j}) = F_{S(X_i,X_j)}(s)\leq \tau$, for a chosen cutoff probability $\tau$.

\subsection{cdf of $S_{i,j}$ and distance $D(X_i,X_j; \rho_{i,j})$}
\label{sec:distiscdf}
\begin{definition}
  \label{defn:cdf_s}
  The density $f_{S_{i,j}}(s)$ of $S_{i,j}$ (in Equation~\ref{eqn:pdf_s}) is used to write the cdf $F_{S_{i,j}}(s)$ of $S_{i,j}$, computed at the disparity value $s$:
  \begin{eqnarray}
&    F_{S_{i,j}}(s) =
  \displaystyle{K\left[\int_{0}^s \left(\sqrt{\frac{2}{\pi}}\exp\left(-\frac{u^2}{2}\right) + u\erf\left(\frac{u}{\sqrt{2}}\right) -u\right) du\right]}& \nonumber \\
  &= \displaystyle{K\left[\frac{(s^2+1)}{2}\erf\left(\frac{s}{\sqrt{2}}\right) -\frac{s^2}{2} +\frac{s}{\sqrt{{2}{\pi}}}\exp(-s^2/2) \right]},&
  \label{eqn:cdf_s}
  \end{eqnarray}
  where $K$ is defined in Equation~\ref{eqn:k}. This holds $\forall i,j\in[M]$.
\end{definition}

\begin{remark}
\label{rem:prop}
We see from Equation~\ref{eqn:cdf_s} that $F_{S_{i,j}}(0)=0$ and $F_{S_{i,j}}(1)=1$, given the definition of $K$ in Equation~\ref{eqn:k}. Also, $F_{S_{i,j}}(s)$ is monotonically increasing with $s\in$[0,1]; continuous in [0,1]; and one-to-one in [0,1]. This holds $\forall i,j,\in[M]$. 
\end{remark}
Next we check if $D(X_i, X_j; \rho_{i,j}) = F_{S_{i,j}}(s)$ is a distance function. 

\begin{proposition}
  \label{prop:disf}
  $D(X_i, X_j; \rho_{i,j}) = F_{S_{i,j}}(s)$ for $s=\vert g_{i,j} - \rho_{i,j}\vert$, is the distance between the $i$-th and $j$-th nodes, given that the absolute partial correlation between $X_i$ and $X_j$ is $\rho_{i,j}$. Here $F_{S_{i,j}}(s)= K\left[{((s^2+1)}/{2})\erf\left({s}/{\sqrt{2}}\right)-{s^2}/{2}\right.\\ \left. + ({s}/{\sqrt{{2}{\pi}}})\exp(-s^2/2)\right]$
  \end{proposition}
\begin{proof}
We condense the notation $D(X_i, X_j; \rho_{i,j})$ to $D_{i,j}$.

We prove that $D_{i,j}$ given by $F_{S_{i,j}}(s)$ that is stated in
Equation~\ref{eqn:cdf_s} is a distance function by proving that it is
non-negative; symmetric; is 0 $\iff$ $X_i=X_j$; and obeys the triangle rule.
\begin{enumerate}
\item[---]Non-negative: $D_{i,j}$ is non-negative. This follows
  since $D_{i,j} = F_{S_{i,j}}(s)$ is a {\it{cdf}}, $\forall s\in[0,1]$, $\forall i,j\in [M]$.
\item[---]Symmetric: $D_{i,j} = D_{j,i}$ by definition of $S_{i,j} := \vert G_{i,j} - \rho_{i,j}\vert$. 
\item[---]$D_{i,j}=F_{S_{i,j}}(s) =0$ if $s=0$, from Equation~\ref{eqn:cdf_s}, $\forall i,j\in [M]$.

  Again, if $D(s)=F_{S_{i,j}}(s)=0$, it implies $s=0$ since $F_{S_{i,j}}(\cdot)$ is one-to-one and monotonic increasing in [0,1]. Therefore, $D(S_{i,j}=s)=0 \iff
  s=0$, $\forall i,j\in[M]$. 

\item[---] $\{S_{i,j}\}_{i \leq j; i,j\in\bV}$ is a 
  set comprising identically distributed variables, each distributed as $F_{S_{i,j}}(\cdot)$ as given in Equation~\ref{eqn:cdf_s}, $\forall i,j\in[M]$. Then for $D_{i,j} = F_{S_{i,j}}(s)$, the triangle rule trivially holds $\forall s\in[0,1]$: $F_{S_{i,j}}(s) + F_{S_{j,k}}(s) \geq F_{S_{i,k}}(s)$ i.e.
  $D_{i,j} + D_{j,k} \geq D_{i,k}$, $\forall i,j \in[M]$, $j < k; j,k\in\{1,\ldots, p\}$; $i < k; i,k\in\{1,\ldots, p\}$.
\end{enumerate}      
\end{proof}

Thus, in the RGG drawn in a probabilistic metric space, the distance function $D(\cdot, \cdot; \cdot)$ is the probability distribution of the disparity variable, as defined in Definition~\ref{defn:cdf_s}.

\subsection{Learning the RGG at the given (partial) correlation matrix $\&$ probability of the graph}
\label{sec:graph}
Since we develop the graph learning in the information paradigm that $\bR=[\rho_{i,j}]$ is known, the disparity variable between $X_i$ and $X_j$ reduces to $S(X_i, X_j) = \vert G_{i,j} - \rho_{i,j}\vert$, $\forall i,j\in[M]$.

At the given
$\rho_{i,j}$, the distance $D_{i,j}$ between the $i$-th and $j$-th nodes is $F_{S_{i,j}}(s)$, with $s=\vert g_{i,j} - \rho_{i,j}\vert$, for $g_{i,j}\in\{0,1\}$, $\forall i,j\in[M]$.

\begin{remark}
\label{rem:circ}
In the RGG, we expect that $D_{i,j} <\tau$, would imply $G_{i,j}=1$. However, this assignment of a value to $G_{i,j}$ then raises a concern: there appears a circularity in our assignment of the value $g_{i,j}$ to the edge variable $G_{i,j}$, using the $g_{i,j}$-dependent inter-nodal distance $D_{i,j}$.

  We avoid any such circularity, by stating that if at the given $\rho_{i,j}$,
  $$D_{i,j}=F_{S_{i,j}}(s=\vert g_{i,j} - \rho_{i,j}\vert) < \tau\implies G_{i,j} = g_{i,j}; {\text{ else }} G_{i,j}= 1 - g_{i,j}$$
in the learnt RGG, $\forall i,j\in[M]$.
\end{remark}

Thus, in the RGG ${\cal G}_{S,\bV}(\bR,\tau)$, we will compute $D_{i,j}=F_{S_{i,j}}(s=1 - \rho_{i,j})$ and check if this computed distance falls short of a chosen $\tau$, or not. If it does, then the edge between the $i$-th and $j$-th nodes exists in the RGG; otherwise it does not. At the same time, if this edge exists, we learn the probability with which this edge exists. This then reflects the SRGG signature of the random graph that we learn, in addition to the graph being an RGG (drawn in a probabilstic space). In order to identify the probability with which an edge exists, the following is undertaken.

At the given $\rho_{i,j}$, we use Rejection Sampling to learn the edge variable that is straddled by the $i$-th and $j$-th nodes of the sought RGG. Such sampling will need to be done from the probability of the edge variable $G_{i,j}$ given the known value $\rho_{i,j}$ of the variable $R_{i,j}$ that represents the absolute partial correlation between $X_i$ and $X_j$. Here, $$f_{S_{i,j}}(s=\vert g_{i,j}-\rho_{i,j}\vert) = \Pr(G_{i,j}=g_{i,j}\vert \rho_{i,j}) f_{R_{i,j}}(\rho_{i,j})/\int_{u=0}^1 f_{R_{i,j}}(u) du,$$
which reduces to $f_{S_{i,j}}(s=\vert g_{i,j}-\rho_{i,j}\vert) = C_{i,j} \Pr(G_{i,j}=g_{i,j}\vert \rho_{i,j})$, given that the density of $R_{i,j}$ is known, and $\rho_{i,j}$ is known, s.t. $C_{i,j} := f_{R_{i,j}}(\rho_{i,j})/[\int_{u=0}^1 f_{R_{i,j}}(u) du]$ is a known constant. Then summing over all values of $G_{i,j}$, at the known absolute (partial) correlation $\rho_{i,j}$ between $X_i$ and $X_j$, the constant
$C_{i,j} = K[(1-\rho_{i,j})\erf\left((1-\rho_{i,j})/{\sqrt{2}}\right) - (1-\rho_{i,j}) + {\sqrt{{2}/{\pi}}}\exp(-(1-\rho_{i,j})^2/2) + \rho_{i,j}\erf\left(\rho_{i,j}/{\sqrt{2}}\right) -\rho_{i,j} + {\sqrt{{2}/{\pi}}}\exp(-(\rho_{i,j})^2/2)]$
where $K$ is defined in Equation~\ref{eqn:k}. Thus, $C_{i,j}$ is known, and hence, $\Pr(G_{i,j}=g_{i,j}\vert \rho_{i,j})$ is known in a closed-form way, $\forall i,j\in[M]$.

We will draw $n_{i,j}$ number of samples of the edge variable
$G_{i,j}$ from its probability mass function conditional on the known
$\rho_{i,j}$. But $\Pr(G_{i,j}=g_{i,j}\vert \rho_{i,j}) =
f_{S_{i,j}}(s)/C_{i,j} = K[\vert
  g_{i,j}-\rho_{i,j}\vert\erf\left(\vert
  g_{i,j}-\rho_{i,j}\vert/{\sqrt{2}}\right) -\vert
  g_{i,j}-\rho_{i,j}\vert +
  {\sqrt{{2}/{\pi}}}\exp(-(g_{i,j}-\rho_{i,j})^2/2)]/C_{i,j}$.

It
follows that $\Pr(G_{i,j}=g_{i,j}\vert \rho_{i,j}) =f_{S_{i,j}}(s)/C_{i,j} \leq 1 \implies f_{S_{i,j}}(s)/K
\leq C_{i,j}/K$, for the positive $K$. In fact, definitions of
$f_{S_{i,j}}(\cdot)$, $K$ and $C_{i,j}$ show that $f_{S_{i,j}}(s)/K <
C_{i,j}/K$ $\forall\rho_{i,j}\in[0,1]$, $\forall i,j\in[M]$.

Hence, in our implementation of Rejection Sampling, we draw $n_{i,j}$
samples of $G_{i,j}$ from the proposal density that we choose to be
the $Bernoulli(0.5)$ density, and scale this by the constant
$C_{i,j}/K$, s.t. $C_{i,j}/K$ envelopes the target $f_{S_{i,j}}(s)/K =
(\vert g_{i,j}-\rho_{i,j}\vert\erf\left(\vert
g_{i,j}-\rho_{i,j}\vert/{\sqrt{2}}\right) -\vert
g_{i,j}-\rho_{i,j}\vert +
{\sqrt{{2}/{\pi}}}\exp(-(g_{i,j}-\rho_{i,j})^2/2))/K$.

We accept the $r$-th proposed sample $g_{i,j}^{(\star)}$ if $f_{S_{i,j}}(\vert g_{i,j}^{(\star)} - \rho_{i,j}\vert)/K \geq (C_{i,j}/K)u $, i.e. if $\vert g_{i,j}^{(\star)}-\rho_{i,j}\vert\erf\left(\vert g_{i,j}^{(\star)}-\rho_{i,j}\vert/{\sqrt{2}}\right) -\vert g_{i,j}^{(\star)}-\rho_{i,j}\vert + {\sqrt{{2}/{\pi}}}\exp(-(g_{i,j}^{(\star)}-\rho_{i,j})^2/2) \geq C_{i,j} u$, where $u$ is the value of $U\sim Uniform[0,1]$. Then the $r$-th sample is denoted $g_{i,j}^{(r)}=g_{i,j}^{(\star)}$. If on the other hand, $g_{i,j}^{(\star)}$ is not accepted, then we set $g_{i,j}^{(r)} = 1- g_{i,j}^{(\star)}$.

We collate these samples into the set $\bP_{i,j}^{(\bR)} :=$ $\{g_{i,j}^{(1)},  g_{i,j}^{(2)},  \ldots, g_{i,j}^{(n_{i,j})}\}$, to define the relative frequency of $G_{i,j}$=1 in this set $\bP_{i,j}^{(\bR)}$ as: 
$$w_{i,j} := \displaystyle{\frac{\sum\limits_{k=1}^{n_{i,j}} g_{i,j}^{(k)}}{n_{i,j}}}.$$
Typically $n_{i,j} = n, \:\: \forall i,j\in[M]$ and we have used $n\sim 10^3$ to $10^4$ in undertaken applications. Thus, the sampling of $g_{i,j}$ at the known $\rho_{i,j}$ is undertaken using Rejection Sampling.
The algorithm used to undertake the sampling is presented in Section~\ref{sec:algo}.

\begin{remark}
  We note that the relative frequency $w_{i,j}$ of achieving $G_{i,j}=1$ in our sample generated via Rejection Sampling using $f_{S_{i,j}}(s)$, is an approximation of $\Pr(G_{i,j}=1\vert \rho_{i,j})$, i.e.
  $$w_{i,j} \approx \Pr(G_{i,j}=1\vert \rho_{i,j}).$$
  \end{remark}

Then at the chosen $\tau\in[0, 1]$, if $F_{S_{i,j}}(s=1-\rho_{i,j}) < \tau$, $G_{i,j}=1$, and this edge between the $i$-th and $j$-th nodes exists in ${\cal G}_{S,\bV}(\bR,\tau)$, with probability $w_{i,j}$. If $F_{S_{i,j}}(s=1-\rho_{i,j}) \geq \tau$, then the edge between the $i$-th and $j$-th nodes exists with probability 0, i.e. $G_{i,j}=0$ then. This holds $\forall i,j\in[M]$.

Thus, we notice that it is possible for us to learn a random graph with the known (partial) correlation matrix $\bR$ of the dataset ${\bf D}$, without imposing any thresholding on any edge, i.e. at $\tau=1$. At the same time, our graph learning can be undertaken while acknowledging the geometric-ness of the graph. In particular, we identify a data-driven optimal $\tau$ in an application given a real dataset (Section~\ref{sec:flood}). 

To compute the adjacency matrix of the RGG ${\cal G}_{S,\bV}(\bR,\tau)$, we first compute the indicator function
\begin{equation}
  \bone_{\bH^{(\tau)}}(F_{S_{i,j}}(1-\rho_{i,j}))), {\text{ where }} \bH^{(\tau)} := \{\delta: \delta < \tau, \delta\in[0,1]\},
  \label{eqn:bone}
\end{equation}
at a chosen $\tau\in[0,1]$.
Then the adjacency matrix of the learnt RGG is
\begin{equation}
  {\bf A} = [w_{i,j}\:\bone_{\bH^{(\tau)}}(F_{S_{i,j}}(1-\rho_{i,j}))].
\label{eqn:adc}
\end{equation}

We define the posterior probability of the RGG of the given dataset ${\bf D}$, using the probability of the edges that are accepted at a chosen $\tau$.
\begin{definition}
Recalling that edges form independently in the RGG ${\cal G}_{S, \bV}(\bR,\tau)$ that we seek of the dataset ${\bf D}$, and that there are no self-loops in this RGG, the posterior probability of this RGG variable, given the known partial correlation matrix $\bR=[\rho_{i,j}]$ of this dataset, is:
\begin{eqnarray}
\pi({\cal G}_{S, \bV}(\bR,\tau)\vert \bR) &=& 
\displaystyle{\prod\limits_{i,j\in[M]} \Pr(G_{i,j}=1\vert\rho_{i,j})\bone_{\bH^{(\tau)}}(F_{S_{i,j}}(1-\rho_{i,j}))}   \\ \nonumber
&\approx& \displaystyle{\prod\limits_{i,j\in[M]} w_{i,j}\bone_{\bH^{(\tau)}}(F_{S_{i,j}}(1-\rho_{i,j}))} 
\label{eqn:graphpost}
\end{eqnarray}
where the set $\bH^{(\tau)}$ is defined in Equation~\ref{eqn:bone}.
\end{definition}

\subsection{Probability metric space that the RGG is drawn in}
\begin{definition}
\label{defn:probsp}
In our work, the triple $\{{\cal X}, D(\cdot), \Delta\}$ is the probabilistic metric
space that the RGG ${\cal G}_{S,\bV}(\bR, \tau)$ is drawn in.

Here the distance function
$$D(X_i,X_j; \rho_{i,j})\equiv D_{i,j} = F_{S_{i,j}}(s)$$ is defined in
Proposition~\ref{prop:disf} for the
{\it{cdf}} $F_{S_{i,j}}(\cdot)\in{\cal F}_{+},\:\forall X_i,X_j\in{\cal X}$, where ${\cal F}_{+}$ is the set of
distributions over positive support.

Lastly, for a given $s\in[0,1]$,
a triangle function $\Delta$ can be defined as the binary operation:
$${\Delta(F_{S_{j,k}}(s), F_{S_{i,k}}(s))} :=
F_{S_{j,k}}(s) + F_{S_{i,k}}(s) \geq F_{S_{k,j}}(s),$$
since $\{S_{a,b}\}_{a,b\in[M]}$ is a set of identically distributed random variables.
Hence, $\Delta(\cdot,\cdot)$
is commutative, associative and takes
$F_{S_{i,i}}(\cdot)=0$ as its identity. Here $S_{i,j}\equiv S(X_i,X_j)$ and the above holds $\forall X_i, X_j, X_k\in{\cal X}$. $i,j,k\in\bV$.
\end{definition}


\subsection{Degree distribution}
\label{sec:point}

Vertex set of RGG ${\cal G}_{S,\bR}(\bV, \tau)$ is $\bV=\{1,\ldots,p\}$. We assume that the
density at point $x$ is $f(x)$.
 For any $i\in\bV$, define (open or close, but bound) Borel
measurable ball $B_{i, a}$, centred at the $i$-th node, with radius $a$, where $a \in [0,1]$, given that radius of a ball in the space of the RGG, is a
probability. 

Let random variable $N(B_{i,a})$ be the number of elements of vertex set
$\bV$ that lie inside $B_{i, a}$, as connected to the $i$-th node at the centre of $B_{i, a}$. We recall that if
$D_{i,j}= F_{S_{i,j}}(s=1 - \rho_{i,j}) = K [({((1-\rho_{i,j})^2+1)}/{2})\erf\left({(1-\rho_{i,j}})/{\sqrt{2}}\right) -{(1-\rho_{i,j})^2}/{2} + $
  $(({1-\rho_{i,j}})/{\sqrt{{2}{\pi}}})\exp(-(1-\rho_{i,j})^2/2)] < \tau$, $G_{i,j}=1$ in our RGG, at the chosen cutoff $\tau$. At the given partial correlation matrix $\bR$ (of the dataset for which this RGG is learnt), the number density of edges that will form between the $i$-th node at the centre of the ball $B_{i,a}$, and a point within distance $x \leq a$ within this ball, is
$$n_i^{(\bR)}(x) := \sum_{j\in\bC_i(x)} f(x) H(\tau - D_{i,j}), {\text{ where }}$$
$$\bC_i(x) := \displaystyle{ \{j: j\neq i, j\in\bV, D_{i,j} \leq x \}}.$$
Here $H(\cdot,\cdot)$ is a Heaviside function.

Then the expectation of the degree of the $i$-th node in the leant RGG 
is:
\begin{eqnarray}
{\mathbb E}[N_i^{(\bR)}] &=& \displaystyle{\int_{x=0}^a
  dx \:\:
  2\pi x n_i^{(\bR)}(x)} \\ \nonumber
&\equiv & \lambda^{(i)}_{a,\tau}(\bR)\pi a^2
\label{eqn:heave}
\end{eqnarray}
where
$$\lambda^{(i)}_{\tau}(\bR) := \displaystyle{\frac{\displaystyle{2\int_{x=0}^a dx\:\:  x n_i^{(\bR)}(x)}}{a^2}},$$
s.t. $${\mathbb E}[N(B_{i,a})] = \lambda^{(i)}_{\tau}(\bR)\pi a^2$$
for the given $a$. If $\tau < a$, there is no contribution to the last integral from $x\in[\tau, a]$; else, the expectation is contributed to by all $x\in[0,a]$.


Thus, the number density of edges within 
a bound region of radius $a$,
centred at point $X_i$ in ${\cal X}$, is given as the
location-dependent $\lambda^{(i)}_{\tau}(\bR)$ for a given partial
correlation matrix $\bR$, at a chosen threshold cutoff probability $\tau$.

Even if the density $f(\cdot)$ is homogeneous, the RGG variable is generated by an inhomogeneous point process. We note that the local intensity of the generative process is a function of the threshold $\tau$, and the correlation matrix of the dataset for which the RGG is learnt.

For a given $X_i$, if $X_j$ and $X_{j^{\prime}}$ are s.t. disparity $S_{i,j} = s$ is less than $S_{i,j^{\prime}}=s^{\prime}$, {\it{cdf}} $F_{S_{i,j}}(s)$ is lower than $F_{S_{i,j}^{\prime}}(s^{\prime})$. Then $D_{i,j} < D_{i,j^{\prime}}$, s.t. the $i$-th node is more likely to be connected to the $j$-th node, than to the $j^{\prime}$-th node, at a given $\tau$. Thus, variables with low values of mutual disparity, could form clusters within the RGG of a the given dataset.
\begin{remark}
$\vert \bC_i(x)\vert$ is approximately the same for all points that lie in the bulk of such a cluster - as is $f(x)$. Thus, nodes that lie in the bulk of a cluster will share similar values of their degrees. On the other hand, for the $i$-th node that lies near the edge of a cluster, $\vert \bC_i(x)\vert$ is lower than if this node were in the bulk. Then, in general, the degree of a node near the edge of a cluster is lower than that in its bulk. Within a ``small'' distance inwards from the edge of a cluster, the degree distribution is likely to show an increasing trend, but how quickly this trend shows up, i.e. quantification of the aforesaid ``small''ness, is affected by the choice of the cutoff $\tau$ used in the learning of the RGG of a given dataset.
\end{remark}





\subsection{Choosing the cutoff probability $\tau$}
Since the learnt graph ${\cal G}_{S, \bV}(\bR, \tau)$ is a random graph-valued variable, we can define its posterior probability $\pi({\cal G}_{S, \bV}(\bR, \tau)\vert \bR)$ given the partial inter-observable correlation matrix $\bR$ of the given dataset. This then allows for an organic identification of the optimal cutoff, at which the most-robust RGG is realised, where said robustness is to changes in $\tau$: the $\tau$ at which the rate of change of the RGG posterior is minimised, i.e. $d(\pi({\cal G}_{S, \bV}(\bR, \tau)\vert \bR))/d\tau$ is minimised, is the optimal cutoff that produces the most-robust RGG. Since we typically compute logarithm of the posterior probability of the RGG, the optimal $\tau$ is implemented as the cutoff, at which the rate of change (with $\tau$) of the logarithm of this posterior probability is minimised. 
  
\subsection{RGG in probabilistic metric space is an SRGG}
\label{sec:issrgg}
We draw an RGG in the probabilistic metric space, and an edge is accepted if the inter-nodal distance falls below a chosen cutoff probability $\tau$, where the accepted edge exists with an identified probability.

\begin{proposition}
  The RGG ${\cal G}_{S, \bV}(\bR, \tau)$, drawn in a probabilistic metric space is an SRGG.
\end{proposition}
\begin{proof}
  At the known partial correlation matrix $\bR=[\rho_{i,j}]$, and for $s=\vert g_{i,j} - \rho_{i,j}\vert$, sampling $G_{i,j}$ from $f_{S_{i,j}}(s)$, $n_{i,j}$ times, produces the set
  $$\bP_{i,j}^{(\bR} = \{g_{i,j}^{(t)}\}_{t=1}^{n_{i,j}},$$
with $w_{i,j} := \sum_{t=1}^{n_{i,j}} g_{i,j}^{(t)}/n_{i,j}$. Then edge probability $\Pr(G_{i,j}=1\vert\rho_{i,j})\approx w_{i,j}$, $\forall i,j,\in[M]$.
  
We have seen above that inter-nodal distance $D_{i,j} \equiv F_{S_{i,j}}(s=1-\rho_{i,j}) < \tau\implies G_{i,j}=1$; else $G_{i,j}=0$.

Then we define the function $\phi:[0,1]\longrightarrow [0,1]$ s.t.
at a chosen $\tau\in[0,1]$, the edge between the $i$-th and $j$-th nodes exists with probability $$\phi(D_{i,j}) := w_{i,j}\bone_{\bH^{(\tau)}}(D_{i,j}),$$ where $$\bH^{(\tau)} := \{\delta: \delta < \tau, \delta\in[0,1]\}.$$ 
This holds $\forall i,j\in[M]$.
  
Thus, in this RGG variable, we have identified the
probability with which an edge exists, where said probability is the function $\phi(\cdot)$ of the distance between the nodes that this edge connects.

An SRGG is a random graph in which the
probability with which an edge exists, is a function of the inter-nodal
distance.

Hence this RGG drawn in a probabilistic metric space, is an SRGG. 
\end{proof}


\section{Posterior density of the inter-observable correlation matrix}
\label{sec:corr}

In this section we present a closed-form posterior probability density of the inter-observable correlation matrix, where $n$ observations of each of $p$ observables $X_1,X_2,\ldots,X_p$ comprise the dataset.
Later in Section~\ref{sec:hypergraph2}, we discuss learning the RGG for vector-valued observables.
 We standardise the data on the variable $X_{i}$
using the mean and standard deviation of the sample comprising the $n$ observations, i.e. the sample $\{x_{i}^{(1)}, x_{i}^{(2)}, \ldots, x_{i}^{(n)}\}$. Such standardisation is undertaken for $\forall i\in\{1,\ldots,p\}$. While values of $X_1, \ldots, X_p$ comprise the dataset ${\bf D}$, data on each of $p$ standardised observables $Z_1, \ldots, Z_p$, comprise the dataset ${\bf D}_s$. The observable $\bX_m$ with standardised components, is denoted $\bZ_{m}$, s.t. we can define vector $\bZ_m = (Z_m^{(1)}, Z_m^{(2)}, \ldots, Z_m^{(p)})^T$, $\forall m \in\{1,\ldots,n\}$. 




Following earlier work, \citep{multigraph, west_2016, wangwest}, we model
the vector-valued observable output at any observational instance, with a vector-valued Gaussian Process. This implies that the joint density of the $n$ realisations of this observable, namely $\bZ_1, \ldots, \bZ_n$, is a matrix Normal density ${\cal MN}({\bf 0}, \bSigma_m^{(S)}, \bSigma_C^{(S)})$. Here, the undertaken standardisation implies that the mean matrix of this density is a null matrix, while its two covariance matrices are rendered correlation matrices. These two correlation matrices are: the $d\times d$-dimensional inter-observable matrix $\bSigma_m^{(S)}$, the $m,m^{\prime}$-th element of which is the correlation between $\bZ_m$ and $\bZ_{m^{\prime}}$, $\forall m,m^{\prime}\in\{1,\ldots,n\}$; and the $p\times p$-dimensional inter-variable correlation matrix $\bSigma_C^{(S)}$, the $i,j$-th element of which is the correlation between $Z_i$ and $Z_j$, $\forall i,j\in\{1,\ldots,p\}$.


\begin{theorem}
\label{theo:marg}
When
the prior on $\bSigma_C^{(S)}$ is Uniform,
the joint posterior probability density of the correlation matrices
$\bSigma_C^{(S)}$ and $\bSigma_m^{(S)}$, given the standardised data
${\bf D}_S$, can be marginalised over all $\bSigma_m^{(S)}$, to yield the
marginal posterior {\it{pdf}} of the inter-observable correlation matrix
given data ${\bf D}_S$, as: $$\pi(\bSigma_C^{(S)}\vert{\bf D}_S)\propto
\displaystyle{{{\Big{\vert}} \bSigma_C^{(S)}{\Big{\vert}}^{-p/2}
    {\Big{\vert}}{\bf D}_S (\bSigma_C^{(S)})^{-1} ({\bf D}_S)^T
    {\Big{\vert}}^{-\frac{m+1}{2}}}},$$
where 
prior on $\bSigma_m^{(S)}$ is
non-informative: $\pi_0(\bSigma_m^{(S)})={\Big{\vert}}
\bSigma_m^{(S)}{\Big{\vert}}^{\alpha}$, with $\alpha=\displaystyle{-\frac{d}{2}
  -1}$, and $\bSigma_C^{(S)}$ is assumed invertible.
\end{theorem}

The proof follows from marginalising over all $\bSigma_m^{(S)}$, using the
uninformative prior on this matrix, as per the theorem statement, and recalling that the posterior {\it{pdf}} $\pi(\bSigma_C^{(S)}\vert{\bf D}_S)$ is obtained for a Uniform
prior on $\bSigma_C^{(S)}$, s.t. it is proportional to the likelihood of
$\bSigma_C^{(S)}$ in ${\bf D}_S$, i.e.
$\pi(\bSigma_C^{(S)}\vert{\bf
  D}_S) \propto
{\cal L}(\bSigma_C^{(S)}; {\bf D}_S)$. Towards the computation of this likelihood, we use the result that $d(\bSigma_m^{(S)}) = \vert \bY\vert^{-(d+1)} d\bY$
\citep{mathai}, where $\bY:= (\bSigma_m^{(S)})^{-1}$. Then invoking
invertibility of ${\bf D}_S (\bSigma_C^{(S)})^{-1} ({\bf D}_S)^T$, the result
follows.

The posterior probability density of the correlation matrix $\bSigma_C^{(S)}$,
given the data ${\bf D}_S$,
as stated in Theorem~\ref{theo:marg}, is reminiscent of the density of the matrix-valued
$t$-distribution, but differences exist.

\section{Inference}
\label{sec:algo}
\begin{remark}
When the inter-observable correlation matrix $\bSigma_C$ of the given multivariate data
${\bf D}$ is known, then the probability of the RGG variable computed with the
known partial correlation matrix $\bR=[\rho_{i,j}]$ - which is itself computed using the known $\bSigma_C$ -
is closed-form. We then learn $G_{i,j}$ (given $\rho_{i,j}$) by undertaking Rejecton
Sampling from the density of the disparity variable $S_{i,j} = \vert G_{i,j}- \rho_{i,j}\vert$,
$\forall i,j\in[M]$.
\end{remark}
In Section~\ref{sec:algo1}, we present the steps for undertaking Rejection Sampling-based learning of the RGG, given
the known correlation matrix of the given data. This approach is referred to as Algorithm~1.

\begin{remark}
When we wish to learn $\bSigma_C$ of data ${\bf D}$, and then learn the RGG
given this learnt correlation matrix, we undertake MCMC-based Bayesian inference on the RGG and
the partial correlation matrix $\bR$ simultaneously, within a
Metropolis-with-2-block-update algorithm. We update $\bR$
(by updating $\bSigma_C$) in the first
block of an iteration of the undertaken MCMC chain, and then in the second block of
the iteration, update the RGG
variable using the updated $\bR$. However, there is no feedback from this second block to the
first block of the subsequent iteration. Thus, this inferential scheme
differs from Metropolis-within-Gibbs. Hence we refer to it as
Metropolis-with-2-block-update.
\end{remark}
The algorithm for the learning of the correlation matrix and the RGG using
MCMC, is shown
in Algorithm~2. A schematic flow of the undertaken inference is in Figure~\ref{fig:flow}.

In the applications we undertake, we do not possess strong prior information
on the unknowns; in such situations, the priors used are retained as
weak. This is manifest in high variances used for the prior probability
density on any unknown. A graph edge parameter has a prior of
$Bernoulli(0.5)$. In applications that are blessed with more information,
stronger priors can be used.

\begin{figure}[!h]
\centering
\includegraphics[width=12cm]{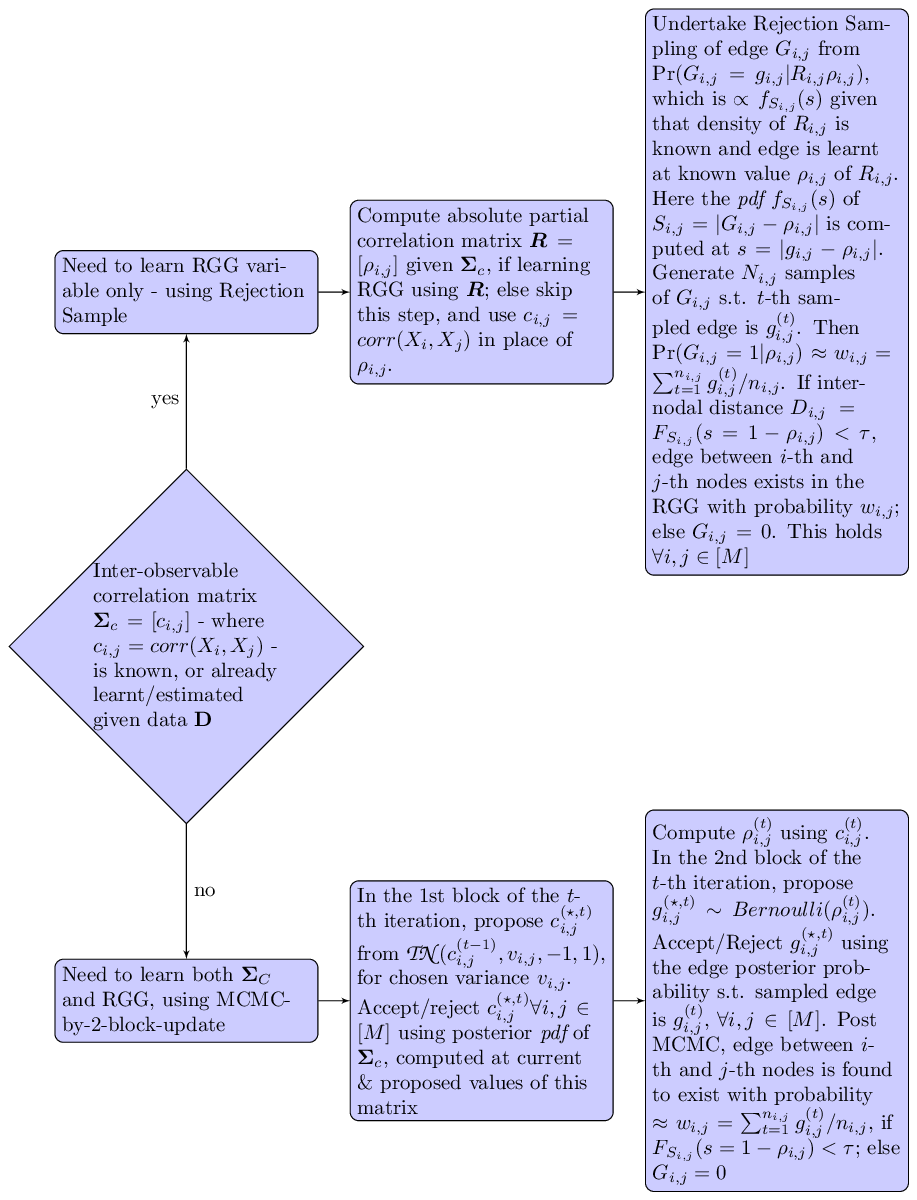}
\vspace{-.2in}
\caption{Flow depicting inference undertaken, when the correlation matrix of
  the given dataset is known or to-be-learnt.} 
\label{fig:flow}
\end{figure}

\subsection{Algorithm~1: implementation of Rejection Sampling}
\label{sec:algo1}
At a chosen $\tau\in[0,1]$, and knowing that the (partial) correlation matrix of the data ${\bf D}$ is $\bR=[\rho_{i,j}]$, the following algorithm is used to perform Rejection Sampling to learn the RGG of ${\bf D}$.


%
\begin{enumerate}
\item In the $k$-th trial of Rejection Sampling, we propose the edge between the $i$-th and $j$-th nodes as $G_{i,j} = g_{i,j}^{(\star, k)}$, from a proposal distribution $p(g_{i,j}) = Bernoulli(0.5)$, and scale it with the constant $C_{i,j}/K$ that envelopes the target function $f_{S_{i,j}}(s=\vert g_{i,j}^{(\star, k)} - \rho_{i,j}\vert)/K$. (See section~\ref{sec:graph} for details). This is equivalent to performing Rejection Sampling from $\Pr(G_{i,j}=g_{i,j}\vert \rho_{i,j}) = f_{S_{i,j}}(s)/C_{i,j}$, using the proposal $p(g_{i,j}) = Bernoulli(0.5)$, $\forall \rho_{i,j}\in[0,1]$. We undertake this $\forall i,j\in[M]$. 
\item Then $(f_{S_{i,j}}(s=\vert g_{i,j}^{(\star, k)} -
  \rho_{i,j}\vert)/K)/((C_{i,j}/K))) \geq u$, implies that we accept the
  $k$-th sample as $g_{i,j}^{(k)}=g_{i,j}^{(\star, k)}$; else we set $g_{i,j}^{(k)}=1 - g_{i,j}^{(\star, k)}$. Here, $U=u$, where $U\sim {\text{Uniform}}[0,1]$. We undertake the Rejection
  Sampling till we have generated $n_{i,j}$ samples that we use to
  populate the set $\bP_{i,j}^{(\bR)}$ as: $\bP_{i,j}^{(\bR)} =
  \{g_{i,j}^{(1)},\ldots, g_{i,j}^{(n_{i,j})}\}.$ The above is
  undertaken for $\forall i,j\in[M]$, with $n_{i,j}=n\sim 10^3$ or
  $10^4$.
\item Compute $w_{i,j} = \sum_{k=1}^{(n_{i,j}} g_{i,j}^{(k)}/n_{i,j}$, $\forall i,j\in[M]$, where $\Pr(G_{i,j}=1\vert \rho_{i,j}) \approx w_{i,j}$.
\item Compute $F_{S_{i,j}}(s=1-\rho_{i,j})$ (using Equation~\ref{eqn:cdf_s}). $F_{S_{i,j}}(s=1-\rho_{i,j}) < \tau \implies$ edge between the $i$-th and $j$-th nodes exists with probability $\Pr(G_{i,j}=1\vert\rho_{i,j})\approx w_{i,j}$. On the other hand, $F_{S_{i,j}}(s=1-\rho_{i,j}) \geq \tau \implies$ $G_{i,j}=0$. This is undertaken $\forall i,j\in[M]$.
\item Thus, learning the value of $G_{i,j}$ $\forall i,j,\in[M]$, will define
  the realisation of the RGG variable ${\cal G}_{S, \bV}(\bR, \tau)$ at a
  chosen $\tau$, given the known partial correlation matrix $\bR$ (of a given
  dataset ${\bf D}$), for which the edge set is;
  $${{E}} = \{g_{i,j} : H(\tau - F_{S_{i,j}}(1-\rho_{i,j})), i,j\in[M]\},$$
  where if the edge exists between the $i$-th and $j$-th nodes, it will do so, with probability approximated by $w_{i,j}$, $\forall i,j\in[M]$.
\end{enumerate}




\subsection{Algorithm~2: implementation of Metropolis-with-2-block-update}

Updating of the correlation matrix is performed in the first block, and
each edge variable is updated in the second block. For the correlation matrix
${\bSigma}_C^{(S)}=[c_{i,j}]$ of the given data ${\bf D}$, in the $t$-th
iteration, $C_{ij}$ is proposed from a Truncated Normal density (${\cal TN}$)
that is left truncated at -1 and right truncated at 1, with an experimentally
chosen variance ($v_{i,j}$), and the current value $c_{i,j}^{(t-1)}$ as the
proposal mean. When we learn the graph with partial correlations - instead of
correlations - we update the partial correlation between $X_i$ and $X_j$ to
$\rho_{i,j}^{(t)}$ - using the value of $C_{i,j}$ that is updated in the first
block of the $t$-th iteration. Then $G_{i,j}$ is proposed in the second block
of this iteration, from a Bernoulli pmf with parameter
$\rho_{i,j}^{(t)}$. We discuss the algorithm below.
\begin{enumerate}
\item In the first block of $t$-th iteration of the MCMC chain,
the $i,j$-th element of the correlation matrix is proposed as:
$$C_{i,j}=c_{i,j}^{\star,t} \sim {\cal TN}(c_{i,j}^{(t-1)}, v_{i,j}, -1, 1),
\:\:\forall i,j\in [M].$$
Let the proposed $C_{i,j}$ value be the $i,j$-th element of the proposed inter-observable correlation
matrix $\bSigma_C^{(\star, t)}$, i.e.
$$\bSigma_C^{(\star, t)} = [c_{i,j}^{\star,t}].$$
\item This proposed covariance matrix may or may not be accepted, (as $\bSigma_C^{(t)}=[c_{i,j}^{(t)}]$),
depending on whether the acceptance ratio leads to acceptance or
rejection. Acceptance occurs if
$$\displaystyle{\frac{\pi(\bSigma_C^{(\star,t)}\vert {\bf D}_S)\times
    \prod\limits_{i, j \in[M]} {\text{{\it{pdf}} of }}{\cal
          TN}(c_{i,j}^{(t-1)}, v_{i,j}, -1, 1)}
  {\pi(\bSigma_C^{(t-1)}\vert{\bf D}_S)\times
    \prod\limits_{i,j\in[M]}{\text{{\it{pdf}} of }}{\cal TN}(c_{i,j}^{(\star, t)}, v_{i,j}, -1, 1)}} \geq u,$$ where $U=u$, with $U\sim Uniform[0,1]$. Here $\bSigma_C^{(t-1)}=
[c_{i,j}^{(t-1)}]$ and posterior probability density of a correlation matrix is
  $\pi(\bSigma_C^{(\cdot)}\vert {\bf D}_S)$ as given in
  Theorem~\ref{theo:marg}. If accepted, we set $\bSigma_C^{(t)} =
  \bSigma_C^{(\star,t)}$; else,
  $\bSigma_C^{(t)}=\bSigma_C^{(t-1)}$. Here, ${\bf D}_S$ is the dataset that
  results from standardisation of ${\bf D}$. 
\item Theorem~\ref{theo:marg} gives the
  posterior of $\bSigma_C^{(S)}$ for Uniform ($Uniform [-1,1]$) prior on
  $c_{i,j}$, $\forall i,j\in\{1,\ldots,p\}$. Thus the likelihood of
  $\bSigma_C^{(S)}$ in the data, is given by this theorem. If in an application, a more
  informative prior on $c_{i,j}$ - than Uniform - is available, then such a
  prior is multiplied with the likelihood, to provide the posterior, that is
  in turn used in the acceptance ratio.    
\item Then
  using this updated ${\bSigma}_C^{(S)}$ matrix, the updated partial
  correlation matrix $\bR^{(t)}=[\rho_{i,j}^{(t)}]$ is
  populated using the updated absolute partial correlations
  $\{\rho_{i,j}^{(t)}\}_{i,j\in \bV}$ that are
  computed using
Equation~\ref{eqn:6}.
\item Subsequently, in the 2nd
block of the $t$-th iteration, the RGG is
updated, using the recently updated absolute partial correlation matrix
$\bR^{(t)}$.
The proposed edge variable connecting the $i$-th to $j$-th vertex is
$$G_{i,j}=g_{i,j}^{(\star,t)}\sim{{{Bernoulli}}}(\rho_{i,j}^{(t)}).$$
We compute the density of the disparity variable $S_{i,j}=\vert G_{i,j} - \rho_{i,j}^{(t)}\vert$, at the proposed edge, where this density $f_{S_{i,j}}(s)$ computed at $s = \vert g_{i,j}^{(\star,t)} - \rho_{i,j}^{(t)}\vert$, is proprtional to $\Pr(G_{i,j}=g_{i,j}^{(\star, t)}\vert \rho_{i,j}^{(t)})$. 
This holds $\forall
i,j\in [M]$.
\item We accept the proposed edge
if the acceptance ratio exceeds or equals $u$ that is the value of $U\sim{ Uniform}[0,1]$,
where acceptance ratio relevant to the 2nd block of the $t$-th iteration is computed at $R_{i,j}=\rho_{i,j}^{(t)}$ as:

\begin{eqnarray}
&\displaystyle{\frac{\prod\limits_{i,j\in [M]} f_{S_{i,j}}(s=\vert g_{i,j}^{(\star,t)} - \rho_{i,j}^{(t)}\vert)}
    {\prod\limits_{i,j\in [M]} f_{S_{i,j}}(s=\vert g_{i,j}^{(t-1)}- \rho_{i,j}^{(t-1)})}}\times &
\nonumber \\
&\displaystyle{\frac{\prod\limits_{i,j\in [M]}
     {\text{{\it{pmf}} of
    }}{{{Bernoulli}}}(\rho_{i,j}^{(t-1)})}
{\prod\limits_{i,j\in [M]}
    {\text{{\it{pmf}} of
    }}{{{Bernoulli}}}(\rho_{i,j}^{(t)})}}& \nonumber
\end{eqnarray}
If accepted, we set $g_{i,j}^{(t)} =
  g_{i,j}^{(\star,t)}$; else,
  $g_{i,j}^{(t)}=g_{i,j}^{(t-1)}$, $\forall i,j\in [M]$. Here the density of the disparity variable is given in Equation~\ref{eqn:pdf_s}.
\item Prior used on $G_{i,j}$ is $Bernoulli(0.5)$.
\end{enumerate}

Values of the edge variable $G_{i,j}$ current in the $n$ post-burnin iterations are collated into the set $\bP_{i,j}^{(\bR)}$, $\forall i,j\in[M]$. Then out of the $n$ samples collected in $\bP_{i,j}^{(\bR)}$, we compute the fraction $w_{i,j}$ of samples that have a value 1, where $w_{i,j}$ estimates the sample mean of $\Pr(G_{i,j}=1\vert\rho_{ij})$. The edge between the $i$-th and $j$-th nodes then exists with probability
$\Pr(G_{i,j}=1\vert\rho_{i,j})\bone_{\bH^{(\tau)}}(F_{S_{i,j}}(1-\rho_{i,j}))$, where the set $\bH^{(\tau)} = \{\delta: \delta < \tau, \delta\in[0,1]\}$. Thus, the RGG ${\cal G}_{S, \bV}(\bR, \tau)$ is learnt, with the probability of the existent edges estimated.





\subsection{When learning large networks}
\label{sec:net}
When our interest is in learning a graphical model on a vertex set with $\vert
\bV\vert =p\gtrsim$20 - as in a large network - the computational cost of
making MCMC-based inference on the off-diagonal elements in the upper (or
lower) triangle of a $p\times p$-dimensional correlation matrix, is
prohibitive, thereby prohibiting computation of $\bR$. Then we use the plugin
estimate ${\hat{c}}_{i,j}$ of $c_{i,j}$, s.t. correlation matrix
$\bSigma_C^{(S)}$ is estimated as $[{\hat{c}}_{i,j}]$. Then the partial
correlation matrix is estimated as ${\hat{\bR}}$, using this estimated
correlation matrix, as long as $p \lesssim 1000$. Thus, such a network is
learnt as the RGG ${\cal G}_{S, \bV}({\hat{\bR}},\tau)$ if $p \lesssim
1000$. However, if $p > 1000$, the cost of inverting the estimated $\bSigma_C^{(S)}$ -
to achieve the corresponding partial correlation matrix - might be desired to
be avoided. In that case, we will learn the RGG using the estimated
correlation matrix, instead of the partial correlation matrix, i.e. as ${\cal
  G}_{S, \bV}(\bSigma_C^{(S)},\tau)$. Then we also reduce the number
$N_{iter}$ of
samples drawn from the marginal of $G_{i,j}$, depending on the value of $p$,
when undertaking Rejection Sampling, $\forall i,j\in[M]$.

We have learnt such a network bearing $> 5000$ nodes, each of which is a human
disease, (Section~\ref{sec:disease}). In this network, any pair of such diseases is correlated by the rank of the extent
of overlap of corresponding phenotypes.

\section{Generalisation to learning an RGG of data on vector-valued observables}
\label{sec:hypergraph2}

Let us consider learning the graph of the dataset ${\bf D}=\{\bx_1^{(k)}, \bx_2^{(k)}, \ldots, \bx_p^{(k)}\}_{k=1}^n$, where the observable $\bX_i=\bx_i^{(\cdot)}\in{\cal X}\subseteq {\mathbb R}^d$, $\forall i\in\{1,\ldots,p\}$. Here we define $\bX_i = (X_{i,1}, \ldots, X_{i,d})^T$. It is also possible that the dataset ${\bf D}$ comprises a varying number of observations of $\bX_i$, distinguished from that of $\bX_j$, for $i,j\in\{1,2,\ldots,p\}$, but here we develop the learning of the graph using the same number ($n$) of observations of each of the $p$ observables. Consequently, the dataset ${\bf D}$ is cuboidal in shape. In fact, we standardise the data ${\bf D}$ to the dataset ${\bf D}_s$, as suggested in Section~\ref{sec:corr}.

Motivated by our RGG learning given scalar observables, we draw the RGG of ${\bf D}_s$ in the probabilistic metric space $\{{\cal X}, D(\cdot), \Delta\}$, with the distance $D(\cdot)$ between a pair of distinct nodes, given by the {\it{cdf}} of the disparity between the observables attached to these nodes, conditional on the absolute (partial) correlation between these variables. The triangle function $\Delta$ is defined in Definition~\ref{defn:probsp}.

This RGG is defined on the vertex set $\bV^{\prime}=\{(1,1),(1,2),\ldots,(1,d),(2,1),$ $\ldots,(2,d),\ldots,(p,1),\ldots,(p,1), \ldots, (p,d)\}$, where the observable $X_{i,m}$ is attached to the $(i,m)$-th node, $\forall i\in\{1,\ldots,p\}$ and $\forall m\in\{1,\ldots,d\}$. The learnt RGG is a flat graph, s.t. edges can exist between nodes $(i,m)$ and $(i,m^{\prime})$, as well as between
nodes $(i,m)$ and $(j,m^{\prime})$ for $m < m^{\prime}; \: m,m^{\prime}\in\{1,\ldots,d\}$ and $i,j\in [M]$, though self-loops are disallowed in our learnt RGGs.

\begin{definition}
In this exercise of learning the RGG with vector-valued variables attached to respective nodes, we invoke the disparity $S_{i,j}^{(m,m^{\prime})}$ between $X_{i,m}$ and $X_{j,m^{\prime}}$, i.e. the disparity between the $m$-th component of $\bX_i$ and the $m^{\prime}$-th component of $\bX_j$, $\forall m,m^{\prime}\in[M]$ and $i <j; i,j\{1,\ldots.d\}$. We invoke the defined disparity between two random variables from above, as: 
$S_{i,j}^{(m,m^{\prime})} = \vert G_{i,j}^{(m,m^{\prime})} - \rho_{i,j}^{(m,m^{\prime})}\vert,$ and $S_{i,j}^{(m,m^{\prime})}$ attains values $\in[0,1]$. Here, the edge variable that connects the $(i,m)$-th node to the $(j,m^{\prime})$-th node is $G_{i,j}^{(m,m^{\prime})}$, and $\rho_{i,j}^{(m,m^{\prime})}$ is the known, absolute (partial) correlation between the variables attached to these two nodes.
Then the distance between these two nodes is $D(X_{i,m}, X_{j,m^{\prime}}) = F_{S_{i,j}^{(m,m^{\prime})}}(s) = K[(s^2+1)/2 \erf(s/\sqrt{2}) -s^2/2 + (s/\sqrt{2\pi})\exp(-s^2/2)]$, as we have seen in Equation~\ref{eqn:cdf_s}, with $K$ defined in Equation~\ref{eqn:k}. In the RGG of dataset ${\bf D}_s$, if $D(X_{i,m}, X_{j,m^{\prime}}) < \tau$, the edge exists between the $(i, m)$-th node (to which $X_{i,m}$ is attached), and the $(j, m^{\prime})$-th node (to which $X_{j,m^{\prime}}$ is attached). If $D(X_{i,m}, X_{j,m^{\prime}}) \geq \tau$, the edge between these two nodes is absent in the RGG.
\end{definition}

This way of learning the RGG for vector-valued observables is then subject to knowledge of the correlation between variable $X_{i,m}$
and variable $X_{j,m^{\prime}}$. Such correlation learning sets the learning of the RGG in this case, different from the learning of the RGG given data on scalar-valued observables.

\subsection{Learning the correlation between $X_{i,m}$ and $X_{j,m^{\prime}}$}
We now discuss learning
the inter-variable given dataset ${\bf D}_s$, in order to learn its RGG. Here ${\bf D}_s$ is built with $n$ observations of $X_{i,m}$, for $i\in\{1,\ldots,p\}$ and $m\in\{1,\ldots,d\}$. Hence, the $pd\times pd$-dimensional $$\bSigma_C \otimes \bSigma_m = [corr(X_{i,m}, X_{j,m^{\prime}})],$$ where
\begin{enumerate}
\item[---]the $p\times p$-dimensional matrix of inter-observable correlations is $\bSigma_C=[corr(\bX_i, \bX_j)]$;
\item[---]the $d\times d$-dimensional matrix of inter-component correlations is $\bSigma_m = [corr(\bW^{(m)}, \bW^{(m^{\prime})})]$,
\item[---]with $\bW$ defined as:
$$\bW^{(m)} = (X_{1,m}, X_{2,m}, \ldots, X_{p,m})^T.$$
\end{enumerate}
Thus, $corr(X_{i,m}, X_{j,m^{\prime}})$ is the $d(i-1)+m, d(j-1)+m^{\prime}$-th element of $\bSigma_C \otimes \bSigma_m$. We need to know $corr(X_{i,m}, X_{j,m^{\prime}})$ to define the absolute (partial) correlation $\rho_{i,j}^{(m,m^{\prime})}$ that allows us to compute the {\it{cdf}} $F_{S_{i,j}^{(m,m^{\prime})}}(s)$ of the disparity between $X_{i,m}$ and $X_{j,m^{\prime}}$, (Equation~\ref{eqn:cdf_s}), and thereby learn if the edge exists between the nodes that these two variables are attached to, at a chosen $\tau$.

\subsubsection{Inter-observable correlation matrix}
The inter-observable correlation $\bSigma_C$ matrix can be estimated
before we start learning the RGG, or could be inferred upon using the
posterior probability density that was reported earlier in
Section~\ref{sec:corr}. This posterior density is
$$\pi(\bSigma_C\vert{\bf D}_S)\propto \displaystyle{{{\Big{\vert}} \bSigma_C{\Big{\vert}}^{-p/2} {\Big{\vert}}{\bf D}_S (\bSigma_C)^{-1} ({\bf D}_S)^T {\Big{\vert}}^{-\frac{n+1}{2}}}}.$$

\subsubsection{Inter-component correlation matrix}
The $m,m^{\prime}$-th element of the inter-component correlation $\bSigma_m$ can be computed by transforming a statistical distance or divergence - such as the Hellinger distance or Kullbeck-Leibler divergence - between the probability of the RGG variable learnt given the data on $\bW^{(m)}$ and the probability of the RGG given observations of $\bW^{(m^\prime)}$. Here, $n$ values
of $\bW^{(m)}$ comprise the dataset ${\bf D}_{\bW^{(m)}} := \{(x_{1,m}^{(i)}, \ldots, x_{p,m}^{(i)})\}_{i=1}^n$, $\forall m\in\{1,\ldots,d\}$.

Let $G_{i,j}^{(m)}=g_{i,j}^{(m)}\in\{0,1\}$ be the edge between the
$(i,m)$-th node that hosts the $m$-th component of the $i$-th
observable, and the $(j,m)$-th node, (to which the
$m$-th component of the $j$-th observable is attached). Again,
let $\rho_{i,j}^{(m)}$ be the absolute (partial) correlation between these
two observables. To compute the probability $\pi({\cal G}_{S,
  \bV}(\bR^{(m)},\tau)\vert \bR^{(m)})$ of the RGG variable ${\cal G}_{S,
  \bV}(\bR^{(m)}, \tau)$, given ${\bf D}_{\bW^{(m)}}$, we recall from
Equation~\ref{eqn:graphpost} that we first
need to compute the probability with which the edge $G_{i,j}^{(m)}$ exists in this RGG. We use Rejection Sampling to draw $n$ samples from the unnormalised density of the disparity between variables $X_{i,m}$ and $X_{j,m}$, where said density is $f_{S_{i,j^{(m)}}}(\vert g_{i,j}^{(m)} - \rho_{i,j}^{(m)}\vert)/K = \vert g_{i,j}^{(m)} - \rho_{i,j}^{(m)}\vert \erf\left(\vert g_{i,j}^{(m)} - \rho_{i,j}^{(m)}\vert/{\sqrt{2}}\right) - \vert g_{i,j}^{(m)} - \rho_{i,j}^{(m)}\vert + {\sqrt{\frac{2}{\pi}}} \exp(-(g_{i,j}^{(m)} - \rho_{i,j}^{(m)})^2/2)$, where $K$ is defined in Equation~\ref{eqn:k}. The fraction $w_{i,j}^{(m)}$ of the samples that is a 1, is the approximation of the probability $\Pr(G_{i,j}^{(m)}=1\vert \rho_{i,j}^{(m)})$. To then acknowledge the effect of the thresholding, we compute the distance $D(X_{i,m}, X_{j,m})$ - between the nodes to which $X_{i,m}$ and $X_{j,m}$ are attached - at $G_{i,j}^{(m)}=1$. This is the {\it{cdf}} of the disparity $S_{{i,j}^{(m)}}$ between these two random variables, computed at $1-\rho_{i,j}^{(m)}$, and is given in Equation~\ref{eqn:cdf_s}. If this {\it{cdf}} falls short of the chosen $\tau$, then edge $G_{i,j}^{(m)} =1$, with probability $w_{i,j}^{(m)}$; else, there is no edge between the nodes to which $X_{i,m}$ and $X_{j,m}$ are attached. Then
\begin{eqnarray}
& \pi({\cal G}_{S,\bV}(\bR^{(m)},\tau)\vert \bR^{(m)}) = 
\displaystyle{\prod\limits_{i,j\in[M]} w_{i,j}^{(m)}\bone_{\bH^{(\tau)}}(F_{S_{i,j}^{(m)}}(1-\rho_{i,j}^{(m)}))} &  \\ \nonumber
\label{eqn:graphbig}
\end{eqnarray}
Similarly, the probability of the RGG variable of the
data ${\bf D}_{\bW^{(m^{\prime})}}$ is computed. The Hellinger distance (or Kullbeck Leibler divergence): $$\delta({\cal G}_{S, \bV}(\bR^{(m)}, \tau),{\cal G}_{S, \bV}(\bR^{(m^{\prime})}, \tau))$$ is computed between $\pi({\cal G}_{S,\bV}(\bR^{(m)},\tau)\vert \bR^{(m)})$ and $\pi({\cal G}_{S,\bV}(\bR^{(m^{\prime})},\tau)\vert \bR^{(m^{\prime})})$.

\begin{definition}
  $\delta({\cal G}_{S, \bV}(\bR^{(m)}, \tau), {\cal G}_{S, \bV}(\bR^{(m^{\prime})}, \tau))$, is computed using the set of probability values that are computed using samples (generated in $n_{i,j}$ trials, with Rejection Sampling) of the edge variables relevant to each RGG, where we normalise each computed probability by the maximal - out of all computed - probability values. Additionally, the inter-observable correlation matrices of datasets ${\bf D}_{\bW^{(m)}}$ and ${\bf D}_{\bW^{(m^{\prime})}}$ can be learnt using the closed-form posterior density of the inter-observable correlation matrix discussed in Section~\ref{sec:corr}. Then the $m,m^{\prime}$-th element of the inter-component correlation $\bSigma_m$ is modelled as:
\begin{equation}
  \bSigma_m := [\exp(-\delta({\cal G}_{S, \bV}(\bR^{(m)}, \tau), {\cal G}_{S, \bV}(\bR^{(m^{\prime})}, \tau)))].
\label{eqn:corrisexpdist}
\end{equation}
\end{definition}  

\subsubsection{Correlation between any two components of any two observables}
Knowing the $d\times d$-dimensional $\bSigma_m$ and the $p\times p$-dimensional $\bSigma_C$, we compute the correlation between the $m$-th component of the $i$-th observable and the $m^{\prime}$-th component of the $j$-th observable, $\forall i,j\in\bV$, and $\forall m,m^{\prime}\in\{1,\ldots,d\}$:
$$\bSigma_C \otimes \bSigma_m = [corr(X_{i,m}, X_{j,m^{\prime}})],$$

\subsection{Learning RGG given inter-observable and inter-component correlation matrices}
The RGG that we learn given the cuboidally-shaped dataset ${\bf D}_s$,
is ${\cal G}_{S, \bV^{\prime}}(\bSigma_C\otimes\bSigma_m, \tau)$. Then
for $(i,m), (j,m^{\prime})\in\bV^{\prime}$, the edge
$G_{i,j}^{(m,m^{\prime})}$ between the $(i,m)$-th node and the
$(j,m^{\prime})$-th node exists, if the inter-nodal distance
$D(X_{i,m}, X_{j,m^{\prime}}) = [\sqrt{2/\pi}\exp(-\vert 1 - \rho_{i,j}^{(m,m^{\prime})}\vert^2/2) + \vert 1 - \rho_{i,j}^{(m,m^{\prime})}\vert \erf((\vert 1 - \rho_{i,j}^{(m,m^{\prime})}\vert)/\sqrt{2})
  -\vert 1 - \rho_{i,j}^{(m,m^{\prime})}\vert] < \tau$; else this edge does not exist. As delineated above, we can again estimate the probability with which this edge exists - if at all - using the relative frequency of an edge sample to be 1, amongst samples drawn using Rejection Sampling, from the {\it{pdf}} of the disparity between the variables $X_{i,m}$ and $X_{j,m^{\prime}}$. 
Thus, the RGG is learnt
given the absolute correlation
$\rho_{i,j}^{(m,m^{\prime})}$ between variables $X_{i,m}$ and
$X_{j,m^{\prime}}$, where $\rho_{i,j}^{(m,m^{\prime})}$ is the
$d(i-1)+m, d(j-1)+m^{\prime}$-th element of $\bSigma_C \otimes
\bSigma_m$. We might know the inter-observable matrix $\bSigma_C$, or
learn it using its closed-form posterior, as delineated in
Section~\ref{sec:corr}. The inter-component matrix $\bSigma_m$ might
again be known, or it can be computed as a transformation - as stated
in Equation~\ref{eqn:corrisexpdist} - of the distance between the
probability of the RGG of the data on the $m$-th component of all
observables, and that of the data on the $m^{\prime}$-th component of
all observables.

\section{Empirical illustration: learning the human disease-phenotype network}
\label{sec:disease}
\noindent
\begin{figure}[!hbt]
\vspace{-.4cm}
\hspace*{-1.5cm}
{
\includegraphics[width=16.3cm,height=15cm]{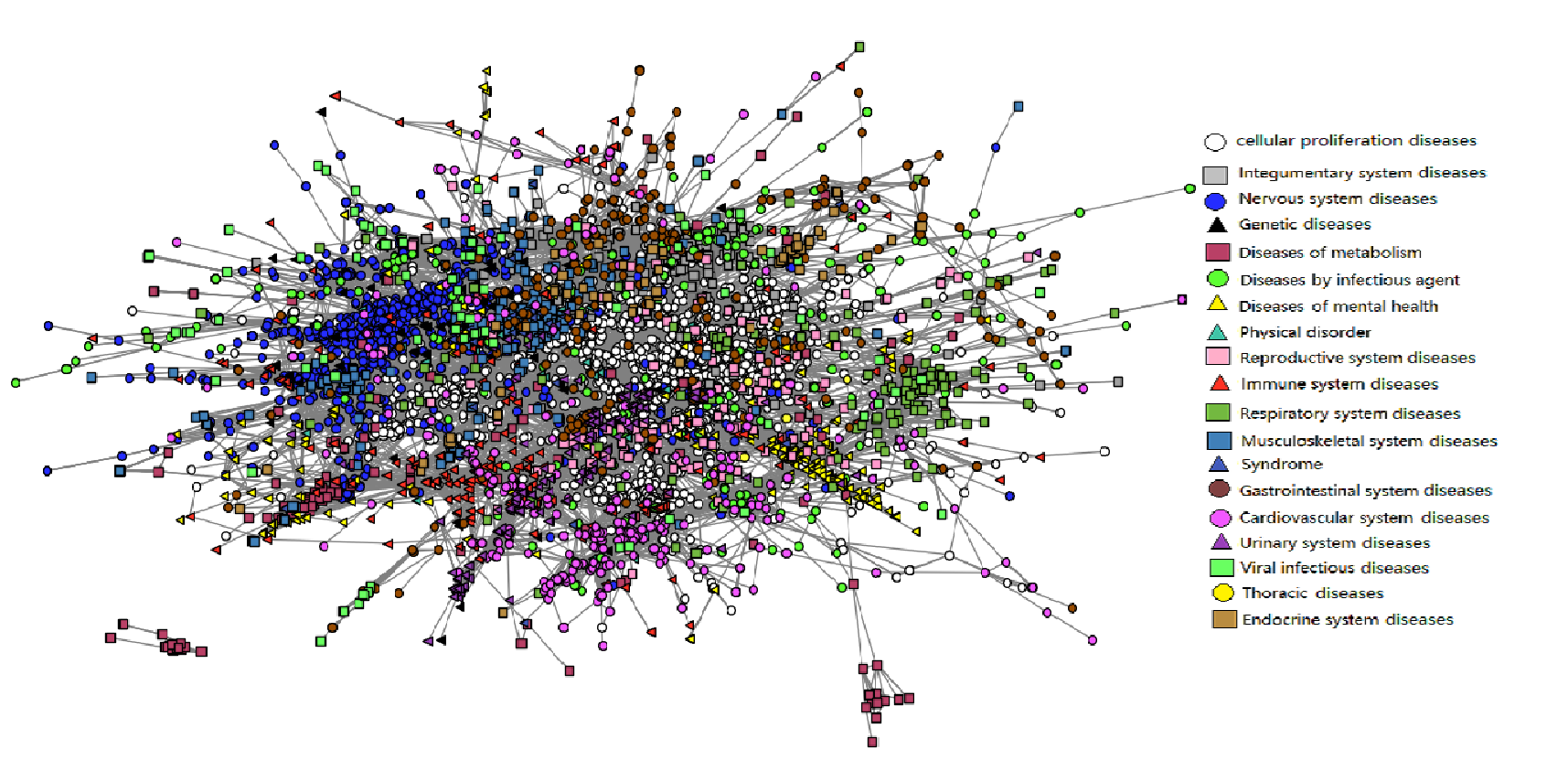}}
\vspace{-.9cm}
\caption{RGG ${\cal G}_{S,\bSigma_C^{(St)}}(\bV,0.1)$, representing the human
  disease-phenotype network that we learn using the computed Spearman rank correlation between
  the rank vectors of a list of phenotypes relevant to a disease, where the phenotype ranking
  reflects semantic relevance of a phenotype to the disease in question
  (quantified by HSG as the NPMI parameter in the ${\bf D}_{DPh}$ dataset). In
  our learnt RGG, $\tau=0.1$.
  Here the vertex set $\bV$ has 8676 elements, but nodes with no edges are discarded from
  this visualised graph, resulting in 6052 diseases (nodes) and 145210 edges
  that are shown this figure. Diseases identified by HSG, to belong to one of
  the 19 given disease class, are presented above in the same colour; the
  colour key identifying these classes, is attached. 
}
\label{fig:disgraph}
\end{figure}
The human disease-phenotype network was learnt by
\cite{hsg} (HSG hereon), by considering the similarity
parameter for each pair of diseases that is an element of an identified set of
diseases in the Human Disease Ontology (DO),
that contains information about
rare and common diseases, and spans heritable, developmental, infectious and
environmental diseases.
Here, the ``similarity parameter'' between one disease
and another, is computed using the ranked vectors of ``normalised pointwise
mutual information'' (NMPI) parameters for the two diseases, where the NMPI
parameter describes the relevance of a phenotype, to
the disease in question. HSG define the NMPI parameter semantically, as the
normalised number of co-occurrences of a given phenotype and a disease in the
titles and abstracts of 5 million articles in Medline. The disease-disease pairwise
semantic similarity parameters -- computed using the degree of overlap in the
relevance ranks of phenotypes associated with each disease -- result in a
similarity matrix, which HSG turn into an inter-disease network based on
phenotypes. They choose from the top-ranking 0.5$\%$ of
inter-disease similarity values. Phenotypes associated with diseases, and
corresponding scoring functions (such as the NPMI), exist in the file
``doid2hpo-fulltext.txt.gz'' at
\url{http://aber-owl.net/aber-owl/diseasephe} \url{notypes}. In fact,
at the site \url{http://aber-owl.net/aber-} \url{owl/diseasephenotypes/data/}, HSG
have uploaded all the data that they have used. 
The file
"doid2hpo-fulltext.txt.gz" available at this site,
contains information about $N_{dis}$ diseases, and
the semantic relevance of each of the $N_{pheno}$ phenotypes to each disease,
as quantified by NPMI parameter values, in addition to other scores such as
$t$-scores and $z$-scores. In this file, $N_{dis}$ is 8676 and $N_{pheno}$ is
19323.

In the phenotypic similarity network between diseases that HSG report,
diseases are the nodes, and the edge between two nodes exists in this
undirected graph, if the similarity between the nodes (diseases) is in the
highest-ranking 0.5$\%$ of the 38,688,400 similarity values. They remove all
self-loops and nodes with a degree of 0. Their network is
presented in \url{http://aber-owl.net/aber-} \url{owl/diseasephenotypes/network/}.
The ``Group
Selector'' function on their visualisation kit, allows for the identification
of 19 clusters in their disease-disease network, with each cluster
corresponding to a disease-class. 
Total number of nodes over their
identified 19 clusters, is 5059; number of edges is 65,795; average node degree$\approx$26.2. 
\begin{figure}[!t]
\centering
\includegraphics[width=11cm]{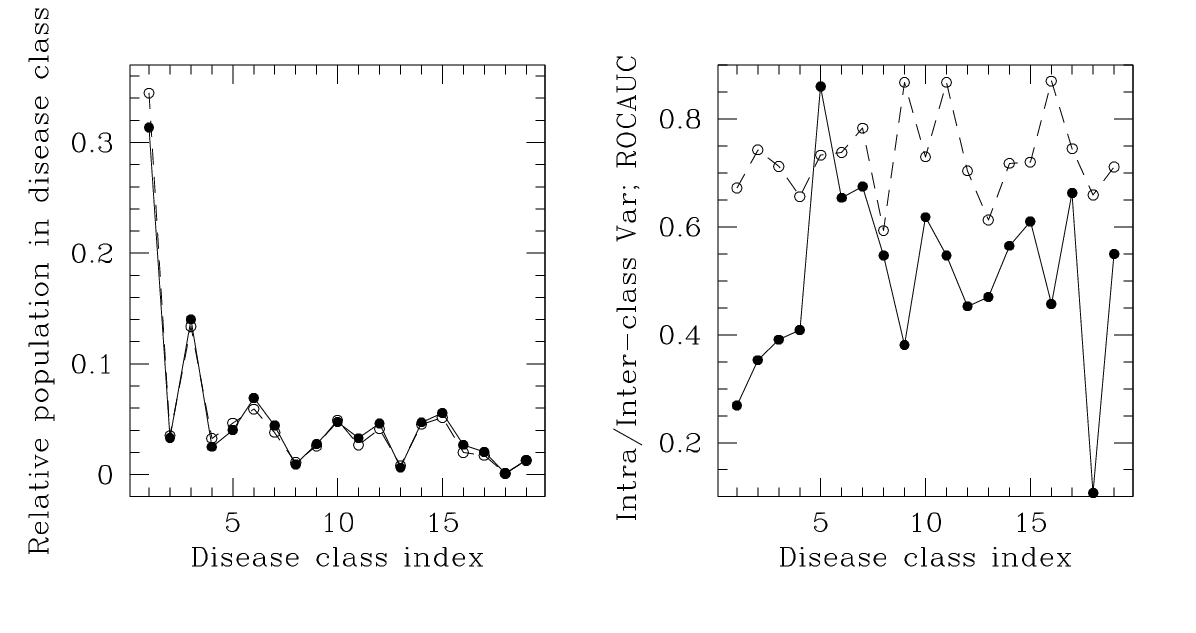}
\vspace{-.3in}
\caption{{\it Left:} comparison of the
  relative number of nodes (diseases) that we recover in each of the 19
  disease classes (in filled circles joined by solid lines), with the relative
  class-membership reported by HSG overplotted as open circles
  threaded by broken lines, {\it Right:}
our computed ratios of the averaged intra-class to inter-class variance for
each of the 19 classes, shown in filled circles; the ROC Area Under Curve
values reported by HSG for each class, is overplotted as open
circles joined by broken lines. The disease class
indices are assigned values 1 to 19; these  
are the following disease classes respectively: cellular proliferation diseases,
integumentary diseases, diseases of the nervous system, genetic diseases,
diseases of metabolism, diseases by infectious agents, diseases of mental
health, physical disorders, diseases of the reproductive system, of the immune
system, of the respiratory system, of the musculoskeletal system, syndromes,
gastrointestinal diseases,cardiovascular diseases, urinary diseases, viral
infections, thoracic diseases, diseases of the endocrine system.
}
\label{fig:compgraph}
\end{figure}

HSG's network then manifests a similarity-structure that is computed using 
available NPMI parameter values.

Our interest is in learning the
disease-disease network as an RGG, with each edge 
existent at a learnt probability. We perform such learning using the
NPMI semantic-relevance data that is made available for each of the $N_{dis}$
number of diseases, by HSG; so $N_{dis}$ is the norm of the vertex set
$\bV$ of our sought RGG. We refer to this human
disease-phenotype data as ${\bf D}_{DPh}$. Using ${\bf D}_{DPh}$, we first
compute the correlation $S_{ij}$ between the $i$-th and $j$-th diseases in $\bV$, for each of
which, information on the ranked (semantic) relevance of each of the
$N_{pheno}$ phenotypes exist, in this given dataset. Upon computation of
pairwise correlations, the RGG for the data ${\bf D}_{DPh}$
is learnt.

We compute the correlation between the $i$-th and $j$-th
diseases in the ${\bf D}_{DPh}$ data, ($i,j\in\{1,\ldots,N_{dis}\}$, $i < j$), in
the following way. We rank the NPMI parameter values - that indicate 
association between the $i$-th
disease and each of the $N_{pheno}$ phenotypes - with phenotypes of 
highest semantic relevance to the $i$-th disease, assigned a rank 1. Let the 
rank vector of phenotypes, by semantic relevance to the $i$-th disease take the value ${\bf{{\mathpzc{r}}_i}}$
and similarly, that for the $j$-th disease
is ${\bf{{\mathpzc{r}}_j}}$. We compute the Spearman rank correlation ${\mathpzc{s}}_{ij}^{(rank)}$,
between vectors ${\bf{{\mathpzc{r}}_i}}$ and ${\bf{{\mathpzc{r}}_j}}$
$\forall\:i,j\in\{1,\ldots,N_{dis}\};\:i < j$.
Spearman rank correlation is preferred to the correlation between vectors of normalised NPMI values, since we intend to correlate the $i$-th disease with the $j$-th disease, depending on how relevant a given list of phenotypes is, to each disease, i.e. on the ranked relevance of phenotypes. 
We learn the network given this correlation,
that is itself computed using data ${\bf D}_{DPh}$ (see
Section~\ref{sec:net} on learning large networks). 
\begin{remark}
  The RGG  visualised in Figure~\ref{fig:disgraph} is a subnet of the full
  network learnt as the RGG
  ${\cal G}_{S,\bSigma_C^{(St)}}(\bV,$ $0.1)$ where $\vert\bV\vert = N_{dis}$, and 
the inter-observable (Spearman rank)
  correlation matrix of data ${\bf D}_{DPh}$ is $\bSigma_C^{(St)}=[{\mathpzc{s}}_{ij}^{(rank)}], \:i < j, \;
  i,j\in\{1,\ldots,N_{dis}\}$, s.t. this visualised graph has 6052 number of nodes,
  each with a non-zero degree, and 145210 edges, so that the average node
  degree is $\approx$24.
This RGG represents our learning of the disease-phenotype network (Figure~\ref{fig:disgraph}).
\end{remark}  

\subsection{Comparing against earlier work}
\noindent
We use Figure~\ref{fig:compgraph} to present comparison of our results to
HSG's, including a comparison between the relative number of nodes
i.e. diseases, in each of the 19 disease classes that HSG classify their
reported network into and our results (shown in the left panel of this figure).
Our learnt RGG tallies very well with the earlier result.  
The right panel of Figure~\ref{fig:compgraph} displays
the ratio of intra-class to inter-class variance of each
disease-class that we identify; value of the area under the Receiver Operating
Characteristic curve (ROCAUC) for each cluster identified by HSG is
overplotted, where the ROCAUC value for the $i$-th cluster can be interpreted
as probability that a randomly chosen node is ranked as more likely to be in
the $i$-th class than in the $j$-th; $i\neq j;\: i,j=1,\ldots,19$.

Thus, our method of learning a large network allows for the learning of the
clustering distribution of the large dataset, for which this network has been
learnt.

\section{Empirical illustration: identifying optimal cut-off}
\label{sec:flood}
\noindent
In this application we consider data on the in-channel water level at 893 river cross-sections, within the wider Humber region in northern England. These water levels are modelled response of
the river system to a collection of 132 simulated storm events, and
 the Environment Agency, U.K. holds proprietorship rights on this data. In this dataset, all storms are characterised by five hydraulic boundary conditions (or parameters), namely the upstream fluvial flow from Aire, Don, Ouse, Trent and the downstream
water level at the mouth of the Humber Estuary. By construction, the dataset includes 132 blocks - each block corresponding to a simulated storm - with a block comprising 893 rows and six columns. s.t. the first five columns of the $i$-th block are populated by the values of the storm parameters relevant to the $i$-th storm, while the last column holds values of the in-channel water level, (referred hereafter as ``water level''),  at each of the 893 river cross-section locations, (referred hereon as ``location''), that are considered; $i\in\{1,\ldots, 132\}$. Then the observations in the first five columns of any block are the same across all 893 rows of this block, though the observed water level values vary from one row to another, within this block.

We define the random variable $X_i$ as the spatially-local water level variable within the Humber region, when the $i$-th storm strikes the region. We denote the $j$-th observation of variable $X_i$ as $x_{i,j}$; this is the water level in the $j$-th location within the region, due to the $i$-th storm striking, for $j\in\{1,\ldots,893\}$. We also define the random variable $W_j$ that represents the water level in the $j$-th location, triggered by a storm event. Then the observed water level in the $j$-th location, due to the $i$-th storm is $x_{i,j}$, i.e. the $i$-th observed value of $W_j$ in the available data is $x_{i,j}$.

We learn one RGG on the vertex set $\bV_{storm} := \{1,\ldots,132\}$, in which, the random variable $X_i$ is attached to the $i$-th node, $\forall i\in\bV_{storm}$. Thus, the dataset that is used to learn this RGG is ${\bf D}_{storm}$ that comprises the observed values of each of $X_1,\ldots, X_{132}$. The $132\times 132$-dimensional correlation function $\bSigma_{storm} = [\vert corr(X_i, X_{i^{\prime}})\vert]$ is used to learn this RGG. We refer to this RGG learnt at the cutoff $\tau$ as ${\cal G}_{S,\bSigma_{storm}}(\bV_{storm},\tau)$.

We learn a second RGG ${\cal G}_{S,\bSigma_{loc}}(\bV_{loc},\tau)$ on the vertex set $\bV_{loc} := \{1,\ldots,893\}$, in which the variable $W_j$ is attached to the $j$-th node $\forall j\in\bV_{loc}$, where the dataset used to learn this RGG is ${\bf D}_{loc}$ that holds observations of $W_1, \ldots, W_{893}$, s.t. the $893\times 893$-dimensional correlation matrix $\bSigma_{loc} = [\vert corr(W_j, W_{j^{\prime}})\vert]$.

We in fact identify the optimal cutoff $\tau^{\star}$ in each dataset, by identifying the $\tau$ at which the rate of change of the logarithm of the posterior probability of the RGG variable is minimised, s.t. the resulting RGG is most resilient to changes in $\tau$. Thus,
$$\tau^{\star} := \displaystyle{\arg\min\left(\frac{d\log\pi({\cal G}_{S,\bSigma_{storm}}(\bV_{storm},\tau)\vert {\bf D}_{storm})}{d\tau}\right)}.$$
We find $\tau^{\star}=0.07808$ in learning given dataset ${\bf D}_{storm}$.

Again, the identification of $\tau^{\star}$ in dataset ${\bf D}_{loc}$ is yields $\tau^{\star}=0.2643$.

\begin{figure}[!h]
  \begin{center}
      $\begin{array}{c c}
          \includegraphics[width=5cm]{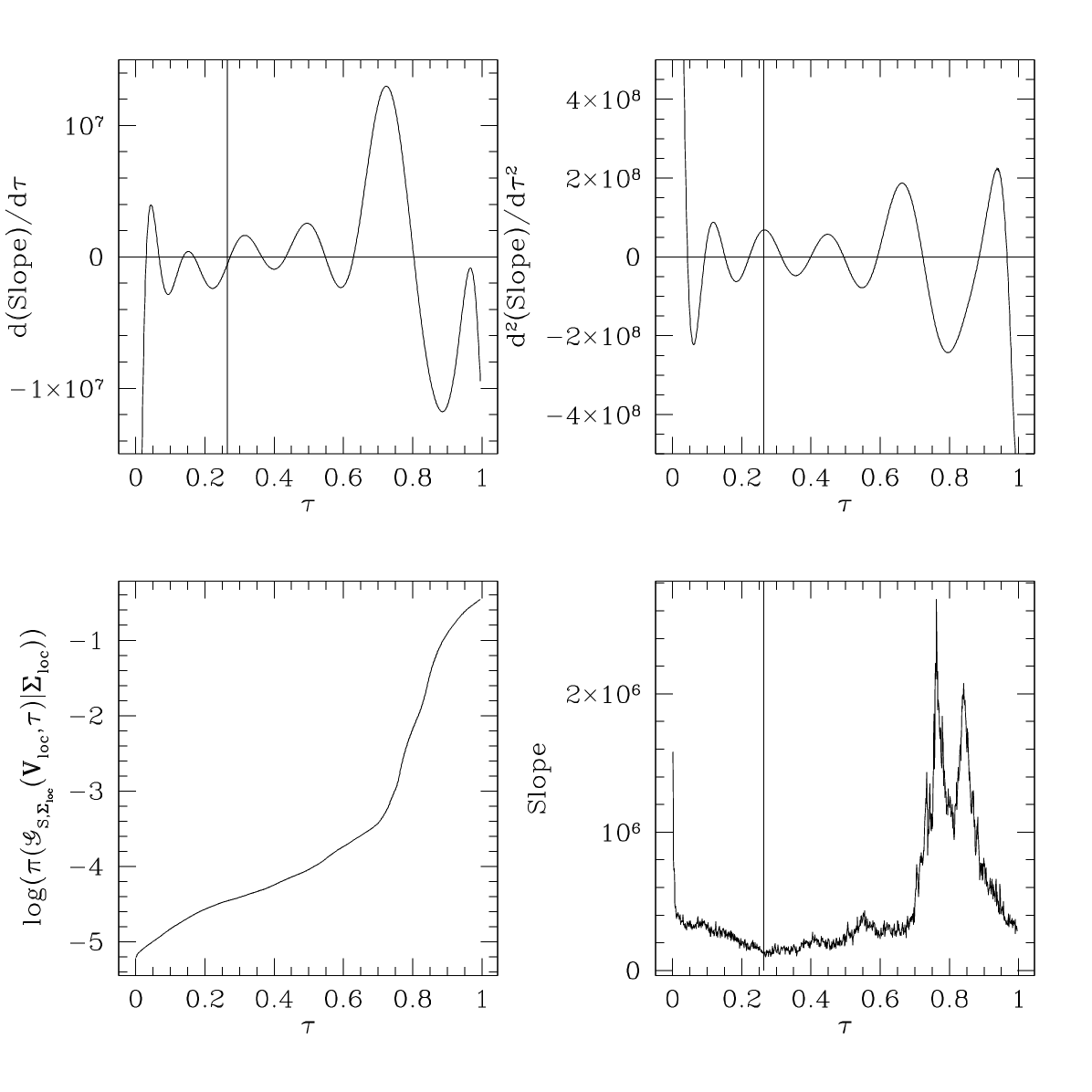} &
          \includegraphics[width=5cm]{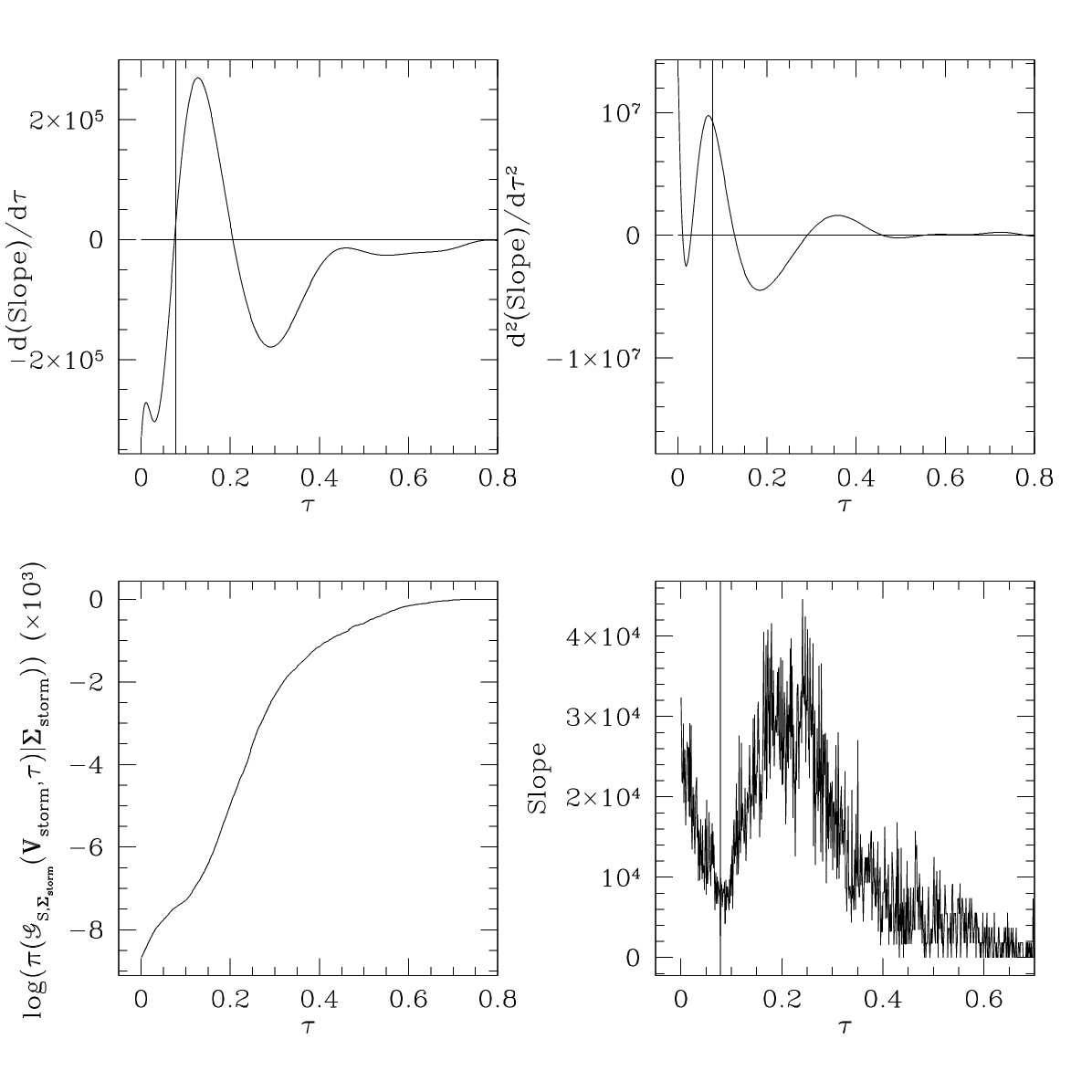}
    \end{array}$
  \end{center}
  \caption{{\it{Block of two left-most panels:}} panel on the bottom left of this block displays the logarithm of the posterior of the RGG variable of dataset ${\bf D}_{loc}$, learnt at varying values of the cut-off probability. The slope of this graph posterior (with respectto $\tau$) is plotted against $\tau$ on the bottom right. The vertical line in black marks the $\tau$ value of 0.2643, at which this slope appears to be minimised. That this $\tau$ value is a minima of the slope, is confirmed by plotting the parametric fit to the numerically-computed first derivative of the slope function with respect to $\tau$ (in the top left), while the fit to the second derivative of the slope function is plotted on the top right of this block. Numerical derivative computation is done via differencing in Python, and fits to the computed data are identified using NumPy's polynomial fitting, (in Python). The vertical lines in these panels respectively confirm the approximately zero value of the first derivative of the slope, as well as the positivity of the second derivative - at $\tau = 0.2643$. {\it{Block of two right-most panels:}} the same as the two left panels discussed above, except here results relevant to the RGG variable of daatset ${\bf D}_{storm}$ are presented. The optimal $\tau$ value for the RGG of this data is identified as $\tau^{\star}=0.07808$, at which the slope of the RGG variable with respect to $\tau$ attains a minima (shown by vertical line in the bottom right panel of this block); where first derivative of the slope is approximately 0, (shown in the top left); and its second derivative is positive, as depicted by the vertical line at $\tau\approx 0.07808$ in the top right panel.}
\label{fig:tau_storm}
\end{figure}



Log posterior of the RGG variables computed at $\tau\in\{0.001, 0.002, \ldots, 0.999\}$, learnt given the correlation between $X_i$ and $X_{i^{\prime}}$, $\forall i,i^{\prime}\in\{1,\ldots,132\}$, are displayed in Figure~\ref{fig:tau_storm}. To illustrate the minimisation of the function that we refer to as $slope(\tau) := \pi({\cal G}_{S,\bSigma_{storm}}(\bV_{storm},\tau)\vert \bSigma_{storm})/d\tau$, we check the closeness of the (fit to the data on the) derivative of $slope(\tau)$ (with respect to $\tau$) to 0, and check if the second derivative is positive. Under these checks, $\tau^{\star}$ for dataset ${\bf D}_{storm}$ is noted to be about 0.07808. The two right-most panels of Figure~\ref{fig:tau_storm} show the same for dataset ${\bf D}_{loc}$ - for which the corresponding RGG is defined on an 893-sized vertex set - to suggest an optimal $\tau$ of about 0.2643. As stated above, in the computation of the derivatives, we use differencing, and fit polynomials (using NumPy's polynomial fitting) to the generated data. The usage of such fitting defines our algorithm of identifying $\tau^{\star}$.

Once the optimal values of $\tau$ are learnt for ${\bf D}_{storm}$ and that for ${\bf D}_{loc}$, we learn the RGG of each of these datasets, at these identified $\tau^{\star}$. The learnt RGGs are visualised in Figure~\ref{fig:locstorm}. Each of the 893 locations that we consider in ${\bf D}_{loc}$, is within the channel of one of four rivers (Aire, Don, Ouse and Trent) that flow in the Humber region. Thus, the RGG in which a node is attached to the local water level variable, has four clusters. This is noted in the RGG learnt given data ${\bf D}_{loc}$.   

\begin{figure}[!h]
  \begin{center}
      $\begin{array}{c c}
          \includegraphics[width=6cm]{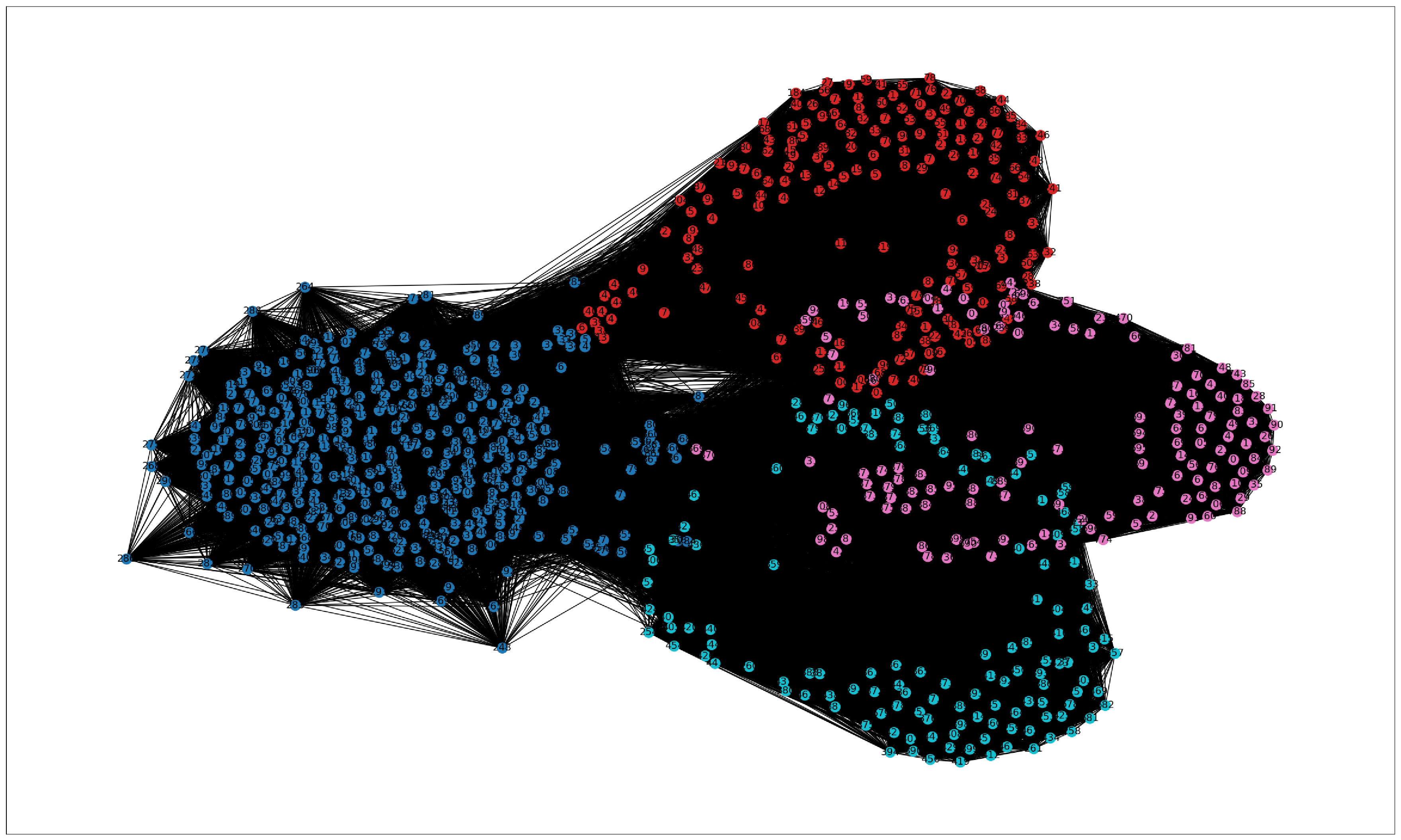} &
          \includegraphics[width=6cm]{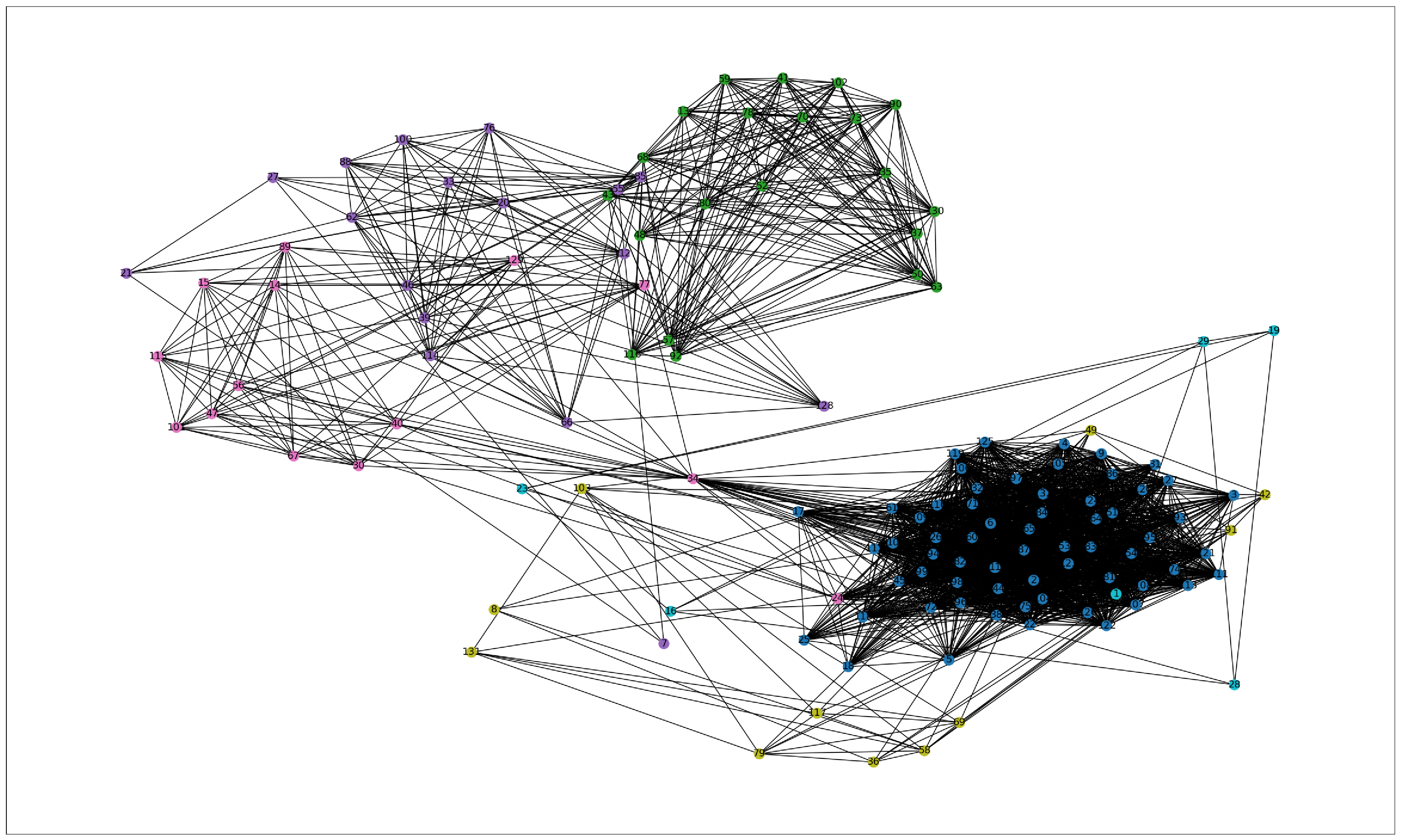}
    \end{array}$
  \end{center}
  \caption{{\it{Left:}} the learnt RGG ${\cal G}_{S,\bSigma_{loc}}(\bV_{loc},\tau^{\star})$, where $\tau^{\star}=0.2643$. The RGG displays four clusters, where the nodes of each cluster are represented in distinct colours. {\it{Right:}} ${\cal G}_{S,\bSigma_{storm}}(\bV_{storm},0.07808)$ learnt at the optimal $\tau$, given dataset ${\bf D}_{storm}$. Nodes that belong to the distinct clusters of this graph, are presented in distinct colours. Membership in a cluster is identified using the clustering function in the Network package in Python.}
\label{fig:locstorm}
\end{figure}

\section{Empirical illustration: degree distribution at varying $\tau$}
\label{sec:users}
We learn RGGs of a dataset comprising binary information, to demonstrate (1) capacity in the RGG-learning to address data on categorical variables, (as well an variables of mixed type); (2) degree distribution, and he effect of the cutoff $\tau$ on this. The purpose of this empirical illustration is to showcase such capcity, while our treatment of the available dataset does not offer any realistic interpretation. 

The data ${\bf D}_{user}$ we use, is available in ``118379821279745746467.feat'', at: {\url{https://snap.stanford.edu/data/ego-Gplus.html}} \citep{snap}, on the $p=500$-dimensional feature vector of $n=699$ users of Google+ Circles which was a core facility of the defunct Google+ Social Network, used to categorise ``friends'`` circles. The presence of any of the $p$ features for any user, is marked by a ``1'' in this dataset, while the absence of the same is marked with a ``0''. We interpret the $p$-dimensional vector of 0s and 1s as a vector comprising $p$ observed values of a binary parameter $X$ that characterises a user, where each observation is made at each of $p$ distinct instances. $X$ can only attain values of 0 and 1. Thus, a variable $X_i\in\{0,1\}$ is attached to the $i$-th node of the RGG of dataset ${\bf D}_{user}$, $\forall i\in\bV_{user}$ where the RGG is defined on the vertex set $\bV_{user}=\{1,\ldots,n\}$. Since $X_i$ is categorical, we use the Cramer's V measure \citep{cramer} to compute $corr(X_i,X_j)$ $\forall i,j\in\bV_{user}$. Then we learn the $n\times n$-dimensional correlation matrix $\bSigma_{user}=[corr(X_i,X_j)]$ and learn the RGG ${\cal G}_{S,\bSigma_{user}}(\bV_{user}, \tau)$, at different $\tau$ values, including at the optimal $\tau$ for this dataset, which is identified as discussed in the previous section, to be about 0.363.

Figure~\ref{fig:user1} presents four RGGs learnt at $\tau$ values of 0.1752, 0.3630, 0.6036, 0.7968; the degree distribution of each of the learnt RGGs is presented below the corresponding RGG. In the RGG learnt at the best cut-off of about 0.363, the two highest values of the degree are the uniqiue two degree values of all points marked respectively in green and blue, in the corresponding RGG. We note this RGG to bear two core clusters, with the remaining points - i.e. non-colourised points bearing low degrees - contributing to the rest of the RGG. The colourised points in the RGG learnt at the optimal $\tau$, are the respective members of the two clusters that distinguish this learnt graph. Thus, almost all members nodes of each cluster bear the same degree, which are also the highest values of the degree amongst all nodes in the graph.

At higher values of $\tau$, we can identify the emergence of further structures in the clustering distribution of the learnt RGG, compared to the cleaner, bimodal clustering distribution that we spot for the RGG learnt at the optimal $\tau$. At these high $\tau$ values, a narrow interval of degree values is identified, s.t. these degrees correspond to nodes close to the edges of the individual clusters. The log-log frequency distribution of degrees in this narrow interval is akin to a straight line with a steep slope. At lower degrees, the log-log degree distribution is noted to be similar to a power-law at $\tau$ values similar to the optimal $\tau$; however, as $\tau$ increases, increasingly more nodes get pulled into the clusters that manifest increasingly more structure, and the degree distribution of the nodes with low degree values then turn Poisson, than power-law.

\begin{figure}[!h]
  \begin{center}
      $\begin{array}{c c c c}
          \includegraphics[width=2.7cm]{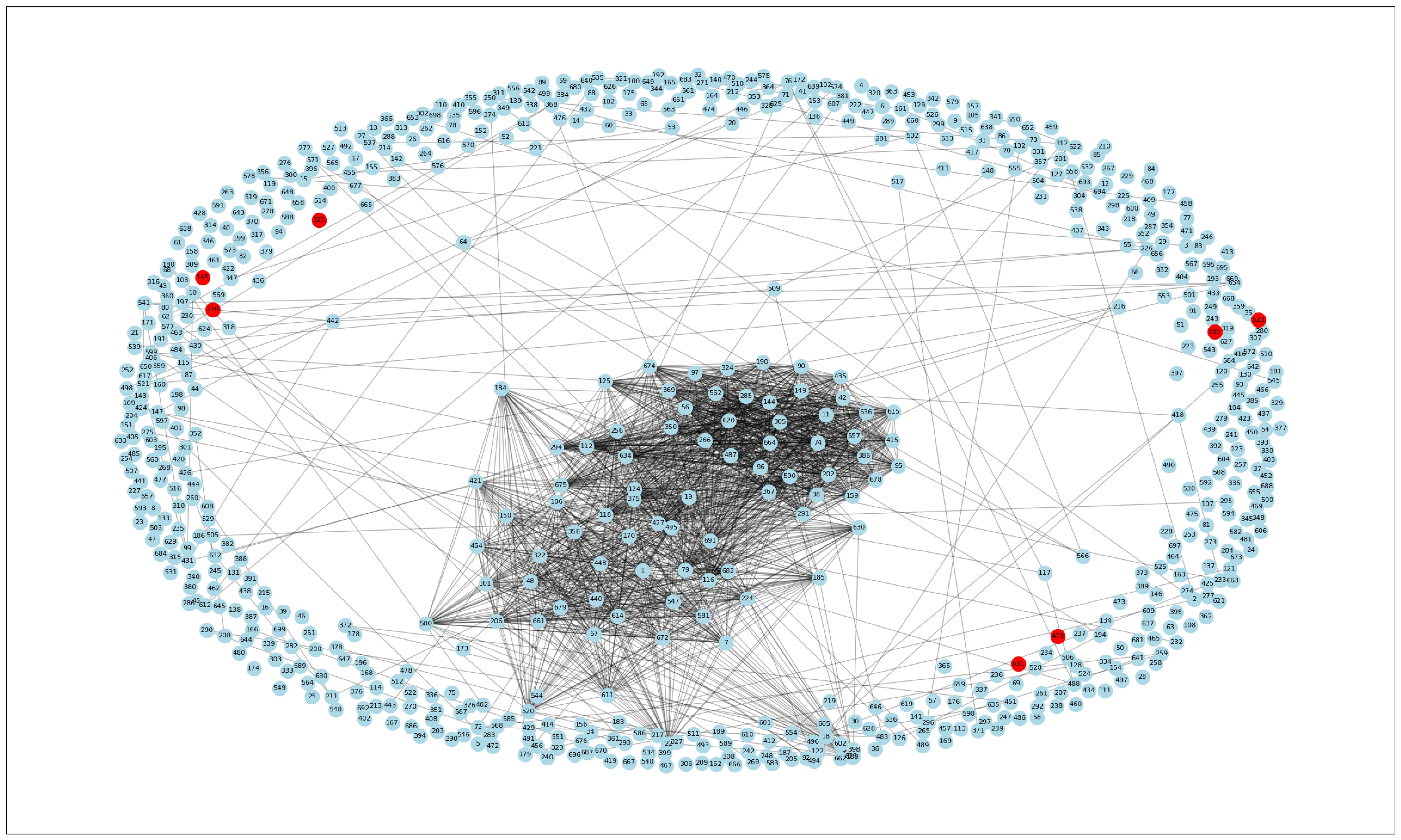} &
          \includegraphics[width=2.7cm]{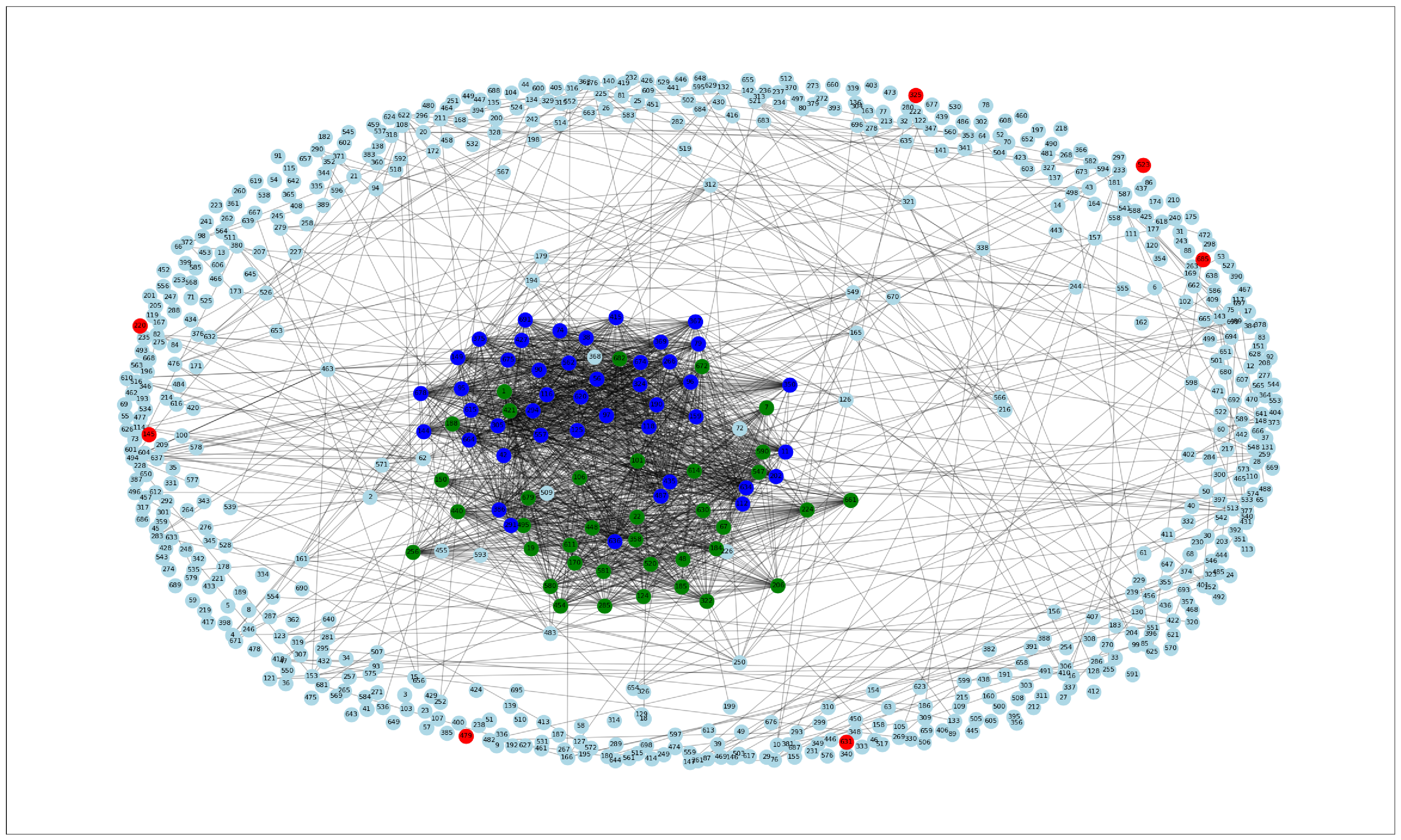} &
          \includegraphics[width=2.7cm]{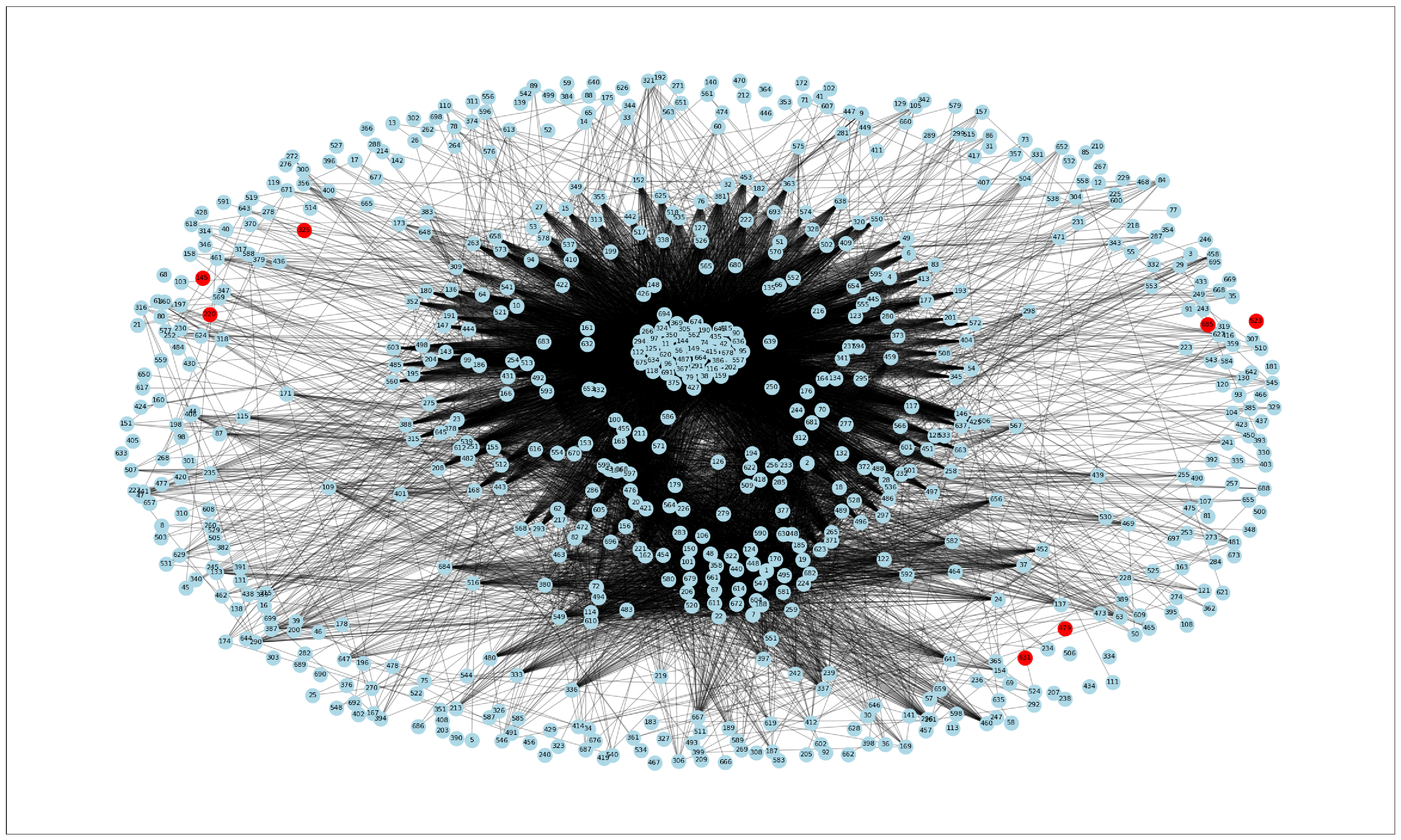} &
          \includegraphics[width=2.7cm]{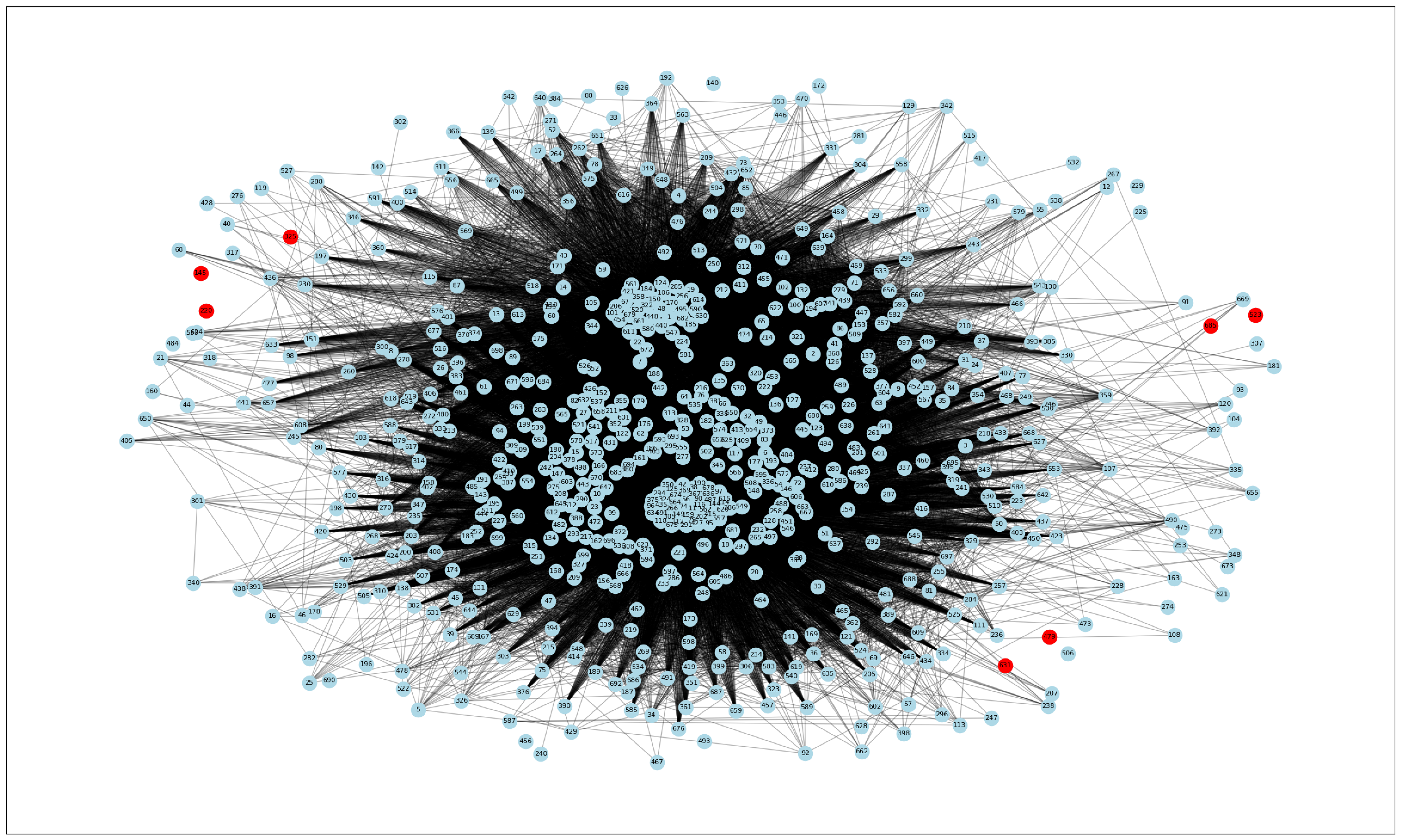}\\
    \end{array}$
  \end{center}
\vspace{-.15in}  
      \hspace*{-.2cm}\includegraphics[width=12.5cm]{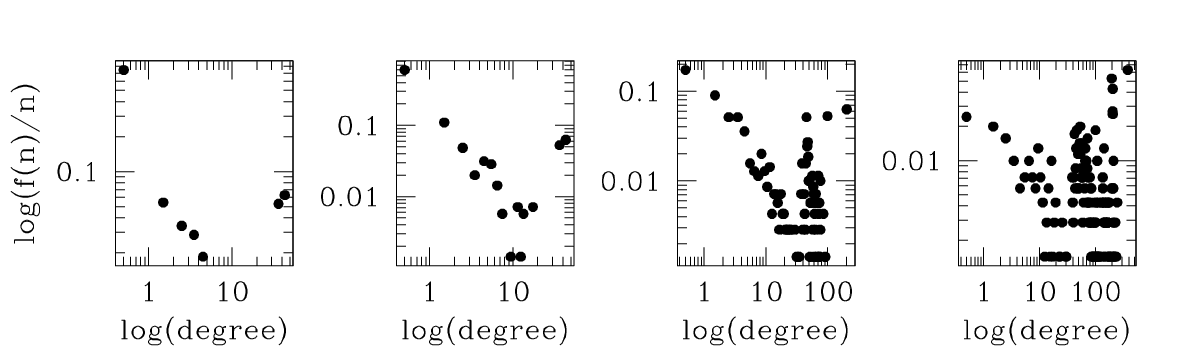}
%
\caption{{\it{Top row:}} RGGs learnt at $\tau$ values of 0.1752, 0.3630, 0.6036, 0.7968, from left to right. Nodes that have a degree of 0, are coloured red in these RGGs. The RGG in the second from left panel, represents the RGG learnt at the optimal $\tau$, given the dataset ${\bf D}_{user}$; nodes in the two modal clusters of this graph are coloured green and blue. {\it{Bottom row:}} figure representing the degree distribution of a learnt RGG, is placed below the panel that presents the learnt RGG. } 
\label{fig:user1}
\end{figure}

\section{Conclusion}
In this paper, we have presented a new learning of a random graph of a given multivariate dataset, as an RGG drawn
in a probabilistic metric space, to result in an SRGG. We forward the closed-form probability of such an undirected inhomogeneous
graph, conditioned on the
correlation structure of the dataset for which the graph is learnt.
We give the
metric of the space that the RGG is drawn in, as the {\it{cdf}} of the disparity variable that we introduce here. Disparity is defined as  the absolute difference between the ``connectedness'' of two nodes, (where said connectedness is given by the mutual edge), and the absolute (partial) correlation between the random variables that are respectively attached to these nodes. Then at the known inter-observable correlation/partial correlation matrix, the edge between these nodes exists in the RGG, if the inter-nodal distance falls short of a chosen cutoff probability $\tau$, i.e. if $\tau$ exceeds the {\it{cdf}} of the disparity computed at edge=1.

To learn the RGG of a given dataset, we use Rejection Sampling to construct a set of samples drawn randomly from the closed-form probability of the edge between a pair of nodes, (conditioned on the correlation between the pair of observables that are attached respectively to the nodal pair). In this set of edge samples, the relative frequency for the edge variable to attain the value of 1, is then proportional to this conditional edge probability. Thus, if the {\it{cdf}} of the disparity computed at edge=1 falls short of a chosen cutoff probability, the edge exists in the RGG, with probability that is (approximately) given by this computed relative frequency.


Learning such an RGG in a probabilistic metric space allows for an easy way of illustrating the correlation structure of a dataset, including data comprising information from vector-valued observables, irrespective of whether the observables are numerical or not. We realise that the optimal $\tau$ at which the RGG is learnt for a given dataset \citep{protein_paper} is also organically realised. The $\tau$ value at which the slope of the probability of the RGG is minimised, is identified as the threshold that produces the graph that is the most robust to changes in the threshold. Thus, the optimal $\tau$ can be identified because we have advanced the closed-form probability of the RGG, given the
inter-observable correlation matrix of the given dataset. In fact, the posterior density of this correlation matrix is closed-form, as we have discussed in Section~\ref{sec:corr}.

While the correlation matrix can be updated using
its closed-form posterior given the dataset at hand, each edge of the graph can be
sampled from the closed-form edge probability, computed using the known
inter-observable correlation matrix of the dataset.
Inference in such a situation is undertaken using a
Metropolis-with-2-block-update scheme. Alternatively, if the correlation
structure of the given dataset is known, then we advocate learning the edges
by undertaking Rejection Sampling from the closed-form edge probability, (given
the known correlation).

Although omitted from the paper, and included in the Supplement, our approach potentially allows for acknowledgement of measurement errors of the observables, in learning the RGG.
It can be demonstrated that the effect of ignoring existent measurement
errors, is to distort the learnt RGG, even in the case  
of a simple, low-dimensional dataset - as demonstrated in Section~1 of the Supplement.

A useful fallout of our RGG learning is the formulation on an inter-graph distance function in the space of such RGGs. Subsequent to the learning of the graph variable given the correlation structure of each of two multivariate datasets, we can define an inter-graph distance as a statistical distance or divergence. Indeed, it is possible to compute the Hellinger distance \citep{matu,juq} between the probabilities of the corresponding pair of random graph variables, given the respective dataset. We have indicated the usage of such an inter-graph distance in Equation~\ref{eqn:corrisexpdist}, and real-world applications of such a distance function have been undertaken \citep{aimed, protein_paper}. In our implementation of the correlation between the $m$-th and $m^{\prime}$-th components of the $i$-th and $j$-th observables, we have used the Hellinger distance between the RGG of the data on the $m$-th component of all observables, and that of the data on the $m^{\prime}$-th component. This inter-graph distance in our work is a generic distance that is informed by the probability distribution of the random graph variable, conditional on the inter-variable correlation matrix of the dataset comprising (noisy) data available on mixed variables in general. The inter-graph distance that we propose, is unlike distances based on counts of node-wise differences \citep{ham1,ham2}, or distances defined for certain types of graphs \citep{mass,dist-regular1,dist-regular2}.

\section*{Appendix~A: Proof of Theorem~\ref{theo:marg}}

\begin{proof}{
Likelihood of correlation matrices
$\bSigma_C^{(S)}$ and $\bSigma_m^{(S)}$, given data ${\bf D}_S$ is matrix
Normal, (see Section~\ref{sec:corr}). By invoking the matrix Normal density, we write the joint posterior of
these correlation matrices using this likelihood, and the priors stated in the
statement of the theorem. The joint posterior probability density of
$\bSigma_C^{(S)}, \bSigma_m^{(S)}$, given data ${\bf D}_S$:
\begin{eqnarray}
\left[\bSigma_C^{(S)}, \bSigma_m^{(S)}\vert{\bf D}_S\right] 
&\propto& 
\displaystyle{
{\cal L}\left(\bSigma_m^{(S)}, \bSigma_C^{(S)}; {\bf D}_S\right)
\left[\bSigma_C^{(S)}, \bSigma_m^{(S)}\right], \quad{\mbox{i.e.}}}\nonumber \\
\left[\bSigma_C^{(S)}, \bSigma_m^{(S)}\vert{\bf D}_S\right] &\propto& 
\displaystyle{
\frac{1}{(2\pi)^{\frac{dp}{2}}
{\Big{|}}\bSigma_C^{(S)}{\Big{|}}^{\frac{p}{2}}
{\Big{|}}\bSigma_m^{(S)}{\Big{|}}^{\frac{d}{2}}}\times}\nonumber \\
&&\displaystyle{
\exp\left[
-\frac{1}{2}tr
\left\{(\bSigma_m^{(S)})^{-1} ({\bf D}_S) (\bSigma_C^{(S)})^{-1} 
({\bf D}_S)^T\right\}
\right]
{\Big{|}\bSigma_m^{(S)}\Big{|}^{-\frac{d}{2} -1}}},
\nonumber \\
&&
\end{eqnarray}
---using prior on 
$\bSigma_m^{(S)}$ to be
$\pi_0(\bSigma_m^{(S)})={\Big{\vert}}
\bSigma_m^{(S)}{\Big{\vert}}^{\alpha}$ where $\alpha=\displaystyle{-\frac{d}{2}
  -1}$, and \\
---using prior on
$\bSigma_C^{(S)}$ to be uniform.

Then marginalising this joint posterior over $\bSigma_m^{(S)}$, we get:
$$
\pi\left(\bSigma_C^{(S)}\vert{\bf D}_S\right) \propto
$$
\begin{equation}
\displaystyle{
\frac{1}{{\Big{|}}\bSigma_C^{(S)}{\Big{|}}^{\frac{p}{2}}}\times} 
\displaystyle{\int\limits_{{\cal M}}
\frac{{\Big{|}}\bSigma_m^{(S)}{\Big{|}}^{-\frac{d}{2}-1}}{{\Big{|}}\bSigma_m^{(S)}{\Big{|}}^{\frac{d}{2}}}
\exp\left[
-\frac{1}{2}tr
\left\{(\bSigma_m^{(S)})^{-1} {\bf D}_S (\bSigma_C^{(S)})^{-1} ({\bf D}_S)^T\right\}
\right]
d(\bSigma_m^{(S)})
}
\label{eqn:2}
\end{equation}
Here $\bSigma_m^{(S)}\in{\cal M}\subseteq{\mathbb R}^{(d\times d)}$. Now, 
\begin{enumerate}
\item[--]let $\bY:= (\bSigma_m^{(S)})^{-1}$. Then $d(\bSigma_m^{(S)}) = \vert \bY\vert^{-(d+1)} d\bY$ \citep{mathai},
\item[--]let $\bV^{-1}:={\bf D}_S (\bSigma_C^{(S)})^{-1} ({\bf D}_S)^T$,
  $\Longrightarrow tr\left[(\bSigma_m^{(S)})^{-1} {\bf D}_S (\bSigma_C^{(S)})^{-1} ({\bf D}_S)^T\right] \equiv 
tr\left[\bV^{-1}\bY\right]$ (using commutativeness of trace), 
\end{enumerate}
so that in Equation~\ref{eqn:2}, we get
\begin{eqnarray}
\pi\left(\bSigma_C^{(S)}\vert{\bf D}_S\right) &\propto&
\displaystyle{
\frac{1}{{\Big{|}}\bSigma_C^{(S)}{\Big{|}}^{\frac{p}{2}}}
\int\limits_{{\cal M}}
{|\bY|^{\frac{d}{2}}}
|\bY|^{\frac{d}{2}+1}\times
\exp\left[
-\frac{1}{2}tr
\left\{\bV^{-1}\bY\right\}
\right]
\vert \bY\vert^{-(d+1)} d\bY
}. \nonumber \\ 
&& 
\label{eqn:3}
\end{eqnarray}
The integral in the RHS of Equation~\ref{eqn:3} represents the unnormalised
Wishart $pdf$ $W_d(\bV, q)$, over all values of the random matrix
$\bY$, where the scale matrix and degrees of freedom of this
$pdf$ are $\bV$ and $q=d+1$ respectively, i.e. $q > d-1$. \\
Thus, integral in the RHS of Equation~\ref{eqn:3} is the integral of the unnormalised $pdf$ of $\bY\sim W_d(\bV, q)$, over the full support of $\bY\left(\equiv\left(\bSigma_R^{(S)}\right)^{-1}\right)$, \\
i.e. the integral in the RHS of Equation~\ref{eqn:3} is
the normalisation of this $pdf$:
$$2^{\frac{qd}{2}} \Gamma_d\left(\frac{q}{2}\right)
\vert\bV\vert^{\frac{q}{2}}\equiv $$
$$2^{\frac{(d+1)(d)}{2}} \Gamma_d\left(\frac{d+1}{2}\right) {\Big{|}}\left({\bf D}_S (\bSigma_C^{(S)})^{-1} ({\bf D}_S)^T\right)^{-1}{\Big{|}}^{\frac{d+1}{2}},$$
i.e. integral on RHS of Equation~\ref{eqn:3} is proportional to
${\Big{|}}\left({\bf D}_S (\bSigma_C^{(S)})^{-1} ({\bf
    D}_S)^T\right)^{-1}{\Big{|}}^{\frac{d+1}{2}},$
\begin{equation}
\text{i.e. }\pi\left(\bSigma_C^{(S)}\vert{\bf D}_S\right) \propto
\displaystyle{
\frac{1}{{\Big{|}}\bSigma_C^{(S)}{\Big{|}}^{\frac{p}{2}}}
{\Big{|}}\left({\bf D}_S (\bSigma_C^{(S)})^{-1} ({\bf
    D}_S)^T\right)^{-1}{\Big{|}}^{\frac{d+1}{2}}}
\label{eqn:notun}
\end{equation}

Now, if
${\bf D}_S (\bSigma_C^{(S)})^{-1} ({\bf D}_S)^T$ is invertible, $${\Big{|}}\left( {\bf D}_S (\bSigma_C^{(S)})^{-1} ({\bf D}_S)^T
\right)^{-1}{\Big{|}}^{\cdot}= {\Big{|}}{\bf D}_S (\bSigma_C^{(S)})^{-1} ({\bf
  D}_S)^T{\Big{|}}^{-{\cdot}}.$$
\begin{enumerate}
\item[--]It is given that $\bSigma_C^{(S)}$ is invertible, i.e. $\left(\bSigma_C^{(S)}\right)^{-1}$ exists.
\item[--]The original dataset is examined to discard rows that are linear
transformations of each other, leading to data matrix ${\bf D}_S$, no two rows of which are linear transformations of each other
\end{enumerate}
$\Longrightarrow$ ${\bf D}_S (\bSigma_C^{(S)})^{-1} ({\bf D}_S)^T$ is
  positive definite, i.e. ${\bf D}_S (\bSigma_C^{(S)})^{-1} ({\bf D}_S)^T$ is
invertible,\\ $\Longrightarrow$
${\Big{|}}\left( {\bf D}_S (\bSigma_C^{(S)})^{-1} ({\bf D}_S)^T \right)^{-1}{\Big{|}}^{(d+1)/2}= {\Big{|}}{\bf D}_S (\bSigma_C^{(S)})^{-1} ({\bf D}_S)^T{\Big{|}}^{-{(d+1)/2}}$. 

Using this in Equation~\ref{eqn:notun}:
\begin{equation}
\pi\left(\bSigma_C^{(S)}\vert{\bf D}_S\right) \propto {\Big{|}}{\bSigma}_C^{(S)}{\Big{|}}^{-{p/2}} {\Big{|}}{\bf D}_S (\bSigma_C^{(S)})^{-1} ({\bf
  D}_S)^T{\Big{|}}^{-{(d+1)/2}}.
\label{eqn:notun2}
\end{equation}
 

}
\end{proof}

\small
\bibliographystyle{ECA_jasa}
\bibliography{irmcmc_sou}
\end{document}